\documentclass{article}

\usepackage[final]{ewrl_2026}

\usepackage{graphicx} 
\usepackage{booktabs}

\usepackage{amsmath}   
\usepackage{amssymb}
\usepackage{amsfonts}  
\usepackage{hyperref}  
\usepackage{enumitem}
\usepackage{amsthm}
\usepackage{xcolor}
\usepackage{colortbl}
\usepackage{wrapfig}
\usepackage{subcaption}
\usepackage{caption}
\usepackage{thmtools}
\usepackage{thm-restate}

\newcommand{\smallsquare}{\scalebox{0.65}{$\blacksquare$}}
\newcommand{\R}{\mathbb{R}}
\newcommand{\universaltask}{{\mathcal{M}_\mathcal{U}}}
\newcommand{\emptytask}{{\mathcal{M}_\varnothing}}
\newcommand{\universalreward}{{r_\mathcal{U}}}
\newcommand{\emptyreward}{{r_\varnothing}}
\newcommand{\gplus}{\mathcal{G}^+}

\title{
A Goal-Set Characterization of Task Composition in the Boolean Task Algebra
}
\author{
    Eduardo Terrés-Caballero$^{1,2}$ \qquad
    Herke van Hoof$^3$\thanks{Correspondence to \texttt{h.c.vanhoof@uva.nl}.}
    \vspace{0.5em}
    \\
    $^1$Informatics Institute, University of Amsterdam \quad
    $^2$Huawei Amsterdam Research Center \\
    $^3$AMLab, University of Amsterdam \\
}

\begin{document}

\maketitle

\begin{abstract}
The Boolean Task Algebra (BTA) provides a principled framework for zero-shot task composition in reinforcement learning by equipping goal-reaching tasks with Boolean operations. We revisit its structural assumptions and formalize a collapse in the space of optimal extended Q-value functions: in deterministic MDPs, every such function is fully determined by the universal and empty tasks. This makes the logarithmic set of base tasks proposed in the original BTA formulation redundant.
Building on this observation, we introduce a goal-set-based composition method that performs logical operations on goal sets and reconstructs composed value functions by selecting slices from the universal and empty value functions. This reduces learning costs for standard BTA and reduces composition time for both BTA and Skill Machines, while preserving policy performance.
Experiments across tabular, visual, function-approximation, and continuous-control domains show that learning additional base tasks does not yield better performance.
Finally, we study the stochastic setting and provide a counterexample showing that this collapse need not hold, that is, optimal composition may require accounting for exponentially many policies in the number of goals.
Code is available at \url{https://github.com/EduardoTerres/GoalSetBTA}.
\end{abstract}

\section{Introduction}
Compositionality is a central principle for developing agents that can reason about and solve multiple tasks within a shared environment. When tasks exhibit structure, it becomes possible to express complex objectives as combinations of simpler ones, thereby enabling reuse of previously acquired behaviors. 
In this direction,
\cite{tasse2020bta} introduce a Boolean Task Algebra (BTA) and show that, under specific assumptions, the space of tasks admits the structure of a Boolean algebra.
By operating on reward functions and their associated extended value functions, the BTA enables exact logical composition of tasks, and provides a principled framework for zero-shot derivation of optimal policies for composite tasks. This algebraic structure further establishes a homomorphism between the task algebra and the algebra of optimal extended Q-value functions, allowing compositions to be expressed directly in extended value-function space.
This algebra requires learning a collection of base tasks whose compositions span the entire algebra. The original formulation proposes learning $\mathcal{O}(\lceil \log | \mathcal{G} | \rceil)$ such tasks in order to guarantee optimal representational completeness.

Following an insight from \cite{tasse2022generalisation} we formalize why this collapse in the BTA happens and what simplifications in terms of computational time for composition can be made to methods that build on this framework (e.g. \cite{tasse2024skillmachinestemporallogic}).
We prove that every optimal extended Q-function can be constructed from only two components: the extended Q-function of the universal task, which assigns the highest terminal reward to every goal, and that of the empty task, which assigns the lowest terminal reward to every goal.

This observation leads to a more direct formulation of the BTA.
Since a task is determined by the set of goals that yield the higher reward, composition can be performed by attending to the desired set of goals.
Extended Q-value functions for composed tasks can then be reconstructed by selecting the appropriate slices from only two learned extended Q-value functions, as long as the two cover the whole space of goals.
This eliminates the need to learn a Boolean basis of tasks, yielding a series of computational advantages.
Additionally, we study how this optimization in composition time would work in the stochastic MDP setting.
In the deterministic case, composing in an optimal manner involves learning goal slices independently and then performing Boolean operations over the set of desired goals, which scales linearly with the size of the desired goal set.
For stochastic MDPs, we provide a counterexample showing that each desired goal set, a subset of the original goal set, may have a different optimal policy in the worst case.
Consequently, this independence between goals disappears, making the number of optimal policies to account for during composition grow with the size of the power set of goals.

\paragraph{Contributions.}
The main contributions of this work are:
\begin{enumerate}
    \item We formalize a representational limitation in the BTA framework stated in \cite{tasse2022generalisation} that drops training cost from $\mathcal{O}(\lceil \log_2 |G| \rceil)$ to constant.
    We empirically show that this improvement does not hinder convergence in either tabular or function approximation settings, as well as show that training additional base tasks does not translate to better performance upon convergence.
    \item We propose a new method for task composition that requires only array lookups and no operations between learned value functions. We empirically confirm the computational gains in tabular, function approximation, visual and continuous-control environments. 
    \item We identify a limitation of deterministic BTA-style composition in stochastic MDPs. While deterministic tasks can be composed by independently combining goal slices, this independence disappears under stochastic transitions and requires operating through an exponential number, with respect to goal set size, of optimal policies.
\end{enumerate}

\section{Related work}

\paragraph{Zero-shot compositional learning in Reinforcement Learning.}
Early work on composition in control theory \citep{todorov2006lsmdp, todorov2009compositionalityofoptimalcontrollaws} demonstrated that Linearly Solvable MDPs admit value and policy decompositions by linearizing the Bellman equation.
Building upon this, \cite{barreto2018successorfeaturestransferreinforcement} introduce Successor Features (SFs) by generalizing successor representations \citep{dayan1993SF}. This framework achieves zero-shot compositionality by factorizing the value function into a dot product of environment dynamics (successor features) and task preferences (reward vectors), enabling the agent to solve novel tasks simply by linearly combining previously learned predictive maps.

\cite{schaul2015uvfa} introduce Universal Value Function Approximator (UVFA) to enable generalization across continuous goal spaces by conditioning value functions on goal representations.
\cite{borsa2018universalsuccessorfeaturesapproximators} later unified this with SF through Universal Successor Features (USFs) and overcame the limitation of discrete task sets inherent in standard SFs by extending the dynamics-reward decomposition to continuous goal spaces by employing a UVFA architecture to approximate the successor features themselves.
\cite{tasse2020bta} build on these ideas by introducing  \textit{extended} Q-value functions, which include a goal dimension in addition to the state and action dimensions.
This extra level of conditioning ultimately allows learning optimal policies.

While the previous methods focus on composing rewards, another stream of research addresses the composition of time and logic. This is primarily handled through Hierarchical Reinforcement Learning (HRL), which simplifies long-horizon tasks by grouping actions into reusable behaviors. Fundamental works like the Options framework \citep{sutton1999options} treat these behavior sequences as single steps, an idea extended by \cite{Andreas2016ModularMR} to chain sub-policies using symbolic sketches, and by \cite{pmlr-v70-vezhnevets17a} to separate high-level planning from low-level execution.
Finally, composition extends to formal specifications, where approaches replace fixed task identifiers with structured logic to enable generalization. By directly grounding Linear Temporal Logic \citep{vaezipoor2021ltl2action, kuric2024planning} or natural language \citep{hermann2017grounded}, these methods allow agents to interpret and execute novel complex constraints zero-shot, simply by parsing the logical structure of the command.

\paragraph{Predecessors and subsequent work on the Boolean Task Algebra.}
The Boolean Task Algebra builds on several works that established how value functions could be manipulated to represent logical operations. \cite{van-niekerk2019composing} demonstrated that while standard RL can bound value composition, Soft Q-learning supports exact disjunctive ($\lor$) composition, as shown by \cite{haarnoja2017reinforcementlearningdeepenergybased}. Parallel work by \cite{haarnoja2018composable} showed that logical conjunction ($\land$) between tasks can be approximated by averaging soft value functions. These compositional mechanisms were further rigorously analyzed by \cite{hunt2019composingentropicpoliciesusing}, who identified theoretical limitations in soft Q-based composition and proposed corrections to ensure reliability.

The Boolean Task Algebra \citep{tasse2020bta} formalizes these binary notions into a cohesive framework, utilizing extended Q-value functions that encode all goal-reaching behaviors to support arbitrary Boolean compositions: disjunction ($\lor$), conjunction ($\land$) and negation ($\neg$).
\cite{tasse2022generalisation} later pointed out the specific structure of optimal value functions but did not analyze further simplifications in the composition process that derive from these.
The authors show that Boolean composition of learned skills enables efficient lifelong transfer and task-space generalization, but in stochastic environments it provides bounded near-optimality guarantees rather than exact compositional optimality.
\cite{tasse2024skillmachinestemporallogic} incorporate these simplifications into the Skill Machines framework, which extends the BTA from purely logical task composition to temporal-logic specifications by using reward-machine-guided \citep{Toro_Icarte_2022rewardmachines} `Skill Machines' that sequence Boolean-composed skills to solve complex long-horizon tasks in a zero-shot manner.

\section{Framework of the Boolean Task Algebra}\label{sec:framework_bta}

\paragraph{Task definition and assumptions.}
Following \cite{tasse2020bta}, we model a task as a Markov Decision Process (MDP),
$
M = (\mathcal{S}, \mathcal{A}, \rho, r)
$,
where $\mathcal{S}$ is the state space, $\mathcal{A}$ is the action space, $\rho(s,a)$ the transition kernel, and $r(s, a)$ the reward function.
We only consider undiscounted MDPs with deterministic transition dynamics and an absorbing set $\mathcal{G} \subseteq \mathcal{S}$, whose elements constitute the goal states of any given task.
We define $\rho, \mathcal{S}$ and $\mathcal{A}$ and assume them fixed for all tasks.
Now, the set of tasks of our interest is defined as
\begin{align}\label{def:task}
\mathcal{M}
= \big\{
    (\mathcal{S}, \mathcal{A}, \rho, r)
    \;\big|\;
    r : \mathcal{S} \times \mathcal{A} \to \R
    \text{ such that }
    &r(s,a) = r_{s,a}
    &&\forall s \notin \mathcal{G}, a \in \mathcal{A} \\
    \text{and } &r(g,a) \in \{ r_{\varnothing}, r_{\mathcal{U}} \}
    &&\forall g \in \mathcal{G}, a \in \mathcal{A}
    \big\}.\notag
\end{align}
Here, $r_{s,a}$ are fixed (task-independent) non-terminal rewards,
and $r_{\varnothing} \le r_{\mathcal{U}}$ denote the only two possible terminal rewards, hence inducing the Boolean nature of the formulation.

\paragraph{Extended Q-value functions.}
Standard value functions are insufficient to compose tasks under the Boolean framework, since they only represent the expected return with respect to the nearest goal \citep{tasse2020bta}. 
To enable reasoning about all possible goals, the extended reward function
$\bar{r} : \mathcal{S} \times \mathcal{G} \times \mathcal{A} \to \R$ is defined as
$
\bar{r}(s,g,a) = r(s,a)
\text{ if } s \notin \mathcal{G} \lor s = g,
\text{ and }
\bar{r}(s,g,a) = \bar{r}_{\min}
\text{ if } s \in \mathcal{G} \setminus \{g\},
$.
Here $\bar{r}_{\min}$ is a large negative penalty ensuring that reaching an undesired goal yields the worst possible return.
The corresponding extended Q-value function is then given by
\[
\bar{Q}^\pi (s, g, a) = \bar{r}(s, g, a) + \int_{\mathcal{S}} \bar{V}^{\pi}(s', g)\,\rho_{(s,a)}(ds'),
\tag{2}
\]
where
$
\bar{V}^{\pi}(s, g) = \mathbb{E}_{\pi}\!\left[\sum_{t=0}^{\infty} \bar{r}(s_t, g, a_t)\right],
$
with optimal counterpart $\bar{Q}^* = \max_\pi \bar{Q}^\pi$.
Intuitively, $\bar{Q}^*(s,g,a)$ measures the value of taking action $a$ in state $s$ when pursuing goal $g$, thus encoding how to reach every goal in the environment.
The standard optimal Q-function for a task $M$ can be recovered as
$Q^*_M(s,a) = \max_{g \in \mathcal{G}} \bar{Q}^*_M(s,g,a)$,
establishing $\bar{Q}^*$ as a richer representation that enables the zero-shot logical composition of tasks.

\paragraph{Boolean algebras of tasks and value functions.}
We endow the set of tasks $\mathcal{M}$ with a Boolean algebraic structure. 
We first define two special tasks: the \emph{universal task} (always maximally rewarding),
$
\mathcal{M}_{\mathcal{U}} = (S, A, \rho, r_{\mathcal{U}}(s,a) = \max_{M \in \mathcal{M}} r_{M}(s,a))
$
and the \emph{empty task} (always minimally rewarding),
$
\mathcal{M}_{\varnothing} = (S, A, \rho, r_{\varnothing}(s,a) = \min_{M \in \mathcal{M}} r_{M}(s,a))
$.
Their corresponding optimal extended value functions, $\bar{Q}^*_{\mathcal{U}}$ and $\bar{Q}^*_{\varnothing}$, represent the maximal and minimal achievable returns.

\smallskip
For any pair of tasks $M_1, M_2 \in \mathcal{M}$, the Boolean operators are defined directly on their reward functions, given that the rest of the conforming elements of the MDP are fixed. Analogously, composition rules for optimal extended Q-value functions are also defined:
\begin{align*}
\text{(Negation $\neg$)} \quad 
&r_{\lnot M}(s,a)
= \big(r_{\mathcal{U}}(s,a) + r_{\varnothing}(s,a)\big) - r_{M}(s,a),\\
&\bar{Q}^*_{\lnot M}(s,g,a)
= \big(\bar{Q}^*_{\mathcal{U}}(s,g,a) + \bar{Q}^*_{\varnothing}(s,g,a)\big)
- \bar{Q}^*_{M}(s,g,a), \\[6pt]
\text{(Disjunction $\lor$)} \quad
&r_{M_1 \lor M_2}(s,a)
= \max\{r_{M_1}(s,a),\, r_{M_2}(s,a)\},\\
&\bar{Q}^*_{M_1 \lor M_2}(s,g,a)
= \max\{\bar{Q}^*_{M_1}(s,g,a),\, \bar{Q}^*_{M_2}(s,g,a)\}, \\[6pt]
\text{(Conjunction $\land$)} \quad
&r_{M_1 \land M_2}(s,a)
= \min\{r_{M_1}(s,a),\, r_{M_2}(s,a)\},\\
&\bar{Q}^*_{M_1 \land M_2}(s,g,a)
= \min\{\bar{Q}^*_{M_1}(s,g,a),\, \bar{Q}^*_{M_2}(s,g,a)\}.
\end{align*}

With these operations and identity elements 
$\mathcal{M}_{\mathcal{U}}$ and $\mathcal{M}_{\varnothing}$ 
for tasks, and 
$\bar{Q}^*_{\mathcal{U}}$ and $\bar{Q}^*_{\varnothing}$ for value functions, \cite{tasse2020bta} show in Theorems 1 and 2 that the tuples
$
(\mathcal{M}, \lor, \land, \lnot, \mathcal{M}_{\mathcal{U}}, \mathcal{M}_{\varnothing})
\quad \text{and} \quad
(\bar{Q}^*, \lor, \land, \lnot, \bar{Q}^*_{\mathcal{U}}, \bar{Q}^*_{\varnothing})
$
form the Boolean Task Algebra and the Boolean value function algebra, respectively.

\paragraph{Homomorphism between algebras.}
A central contribution of \cite{tasse2020bta} is establishing a homomorphism between the algebra of tasks and value functions.
That is, the following constitutes a homomorphism,
$
\mathcal{F} : \mathcal{M} \longrightarrow \bar{Q}^*
$
with
$
\mathcal{F} (M) = \bar{Q}^*_{M},
$
where
$
\bar{Q}^* = \{\bar{Q}^*_M | M \in \mathcal{M}\}
$ is the set of possible optimal extended Q-value functions.
Ultimately, this provides a way for composing value functions, where given $M_1, M_2 \in \mathcal{M}$,
\[
\bar{Q}^*_{\lnot M} = \lnot \bar{Q}_M^*, \quad
\bar{Q}^*_{M_1 \lor M_2} = \bar{Q}_{M_1}^* \lor \bar{Q}_{M_2}^*, \quad
\bar{Q}^*_{M_1 \land M_2} = \bar{Q}_{M_1}^* \land \bar{Q}_{M_2}^*.
\]
This result formally connects logical task composition to the algebraic composition of optimal extended value functions, implying that once the agent has learned $\bar{Q}^*$ for a small set of base tasks, it can derive the optimal policy for any Boolean combination of them without further learning.
The original formulation proves this result for a broader set of tasks that allow for more than two terminal rewards $r_\varnothing$, $r_\mathcal{U}$ but keep the other restrictions present in $\mathcal{M}$.

\paragraph{BTA in Skill Machines.}
Skill machines \citep{tasse2024skillmachinestemporallogic} extends the BTA to handle tasks specified in Linear Temporal Logic (LTL), a formal language that combines Boolean operators with temporal operators such as $\mathbf{F}$ (eventually) and $\mathbf{G}$ (always) to express ordered sequences of goals.
Let $\mathcal{S}$ denote the environment state space, $\mathcal{P}$ the set of propositional symbols representing high-level environment features, $L: \mathcal{S} \to 2^{\mathcal{P}}$ the labelling function assigning truth values to environment states, $G = 2^{\mathcal{P} \cup \mathcal{C}}$ the set of absorbing goal states, and $\sigma_u$ the Boolean expression over $\mathcal{P} \cup \mathcal{C}$ assigned to skill machine state $u$ via planning over the reward machine.
To support this, skill primitives $Q^*_p(\langle s, c \rangle, g, \langle a, a_\tau \rangle)$ are defined over an augmented state space $\mathcal{S}_G = (\mathcal{S} \times 2^\mathcal{C}) \cup 2^{\mathcal{P} \cup \mathcal{C}}$, where $\mathcal{C}$ tracks violated constraints and $a_\tau \in \{0,1\}$ is a termination action.
Since all primitives share the same state space and dynamics, the BTA operators apply directly to yield $Q^*_\sigma$ for any Boolean expression $\sigma$ over $\mathcal{P} \cup \mathcal{C}$.
A skill machine is then a finite state machine $\langle U, u_0, \delta_u, \delta_Q \rangle$ where each state $u \in U$ executes a spatially composed skill $\delta_Q(u) = \max_{g \in G} Q^*_{\sigma_u}(\langle s, c \rangle, g, \langle a, 0 \rangle)$ and transitions to $u' = \delta_u(u, L(s'))$ upon satisfying the required propositions, thereby chaining Boolean-composed skills sequentially to solve any LTL-specified task in a zero-shot manner.

\section{An alternative method for composition}
\label{section:extension_bta}
Having introduced the Boolean Task Algebra framework we now present a series of limitations regarding the representational capacity of this framework.

\cite{tasse2022generalisation} note a collapse in the optimal extended value functions under the BTA framework and state that the optimal extended Q-value function can only take two forms, that of the universal or empty value functions.
However, they do not formalize this phenomenon nor its implications for the framework.
We formalize this issue and further argue that, as a consequence, it is not necessary to learn the full set of base tasks defined in the BTA framework presented, see Section~\ref{sec:framework_bta}.
In particular, the assumptions restrict the space of possible tasks and optimal Q-value functions such that it suffices to learn the value functions of two tasks, regardless of the number of goals in $\mathcal{G}$.
Proposition~\ref{lemma1_cor1} embodies this limitation by showing that the extended Q-value function can be written as the sum of rewards accumulated until reaching a terminal state ($G^{*}_{s:g,a}$, which is shared across tasks) plus the reward for reaching the goal ($\bar{r}_{M}(s', g, a')$).
In Section~\ref{section:experiments}, we empirically investigate whether learning a logarithmic number of base tasks nevertheless improves performance sufficiently to justify the additional training cost.

\begin{restatable}{proposition}{propSharedG}\label{lemma1_cor1}
    \textit{(Corollary 1 from \cite{tasse2020bta})}
    Denote $G^{*}_{s:g,a}$ as the sum of rewards starting from $s$ and taking action $a$ up until, 
    but not including, $g$. Then let $M \in \mathcal{M}$ and $\bar{Q}^{*}_{M}$ be the extended 
    $Q$-value function. Then for all $s \in \mathcal{S}$, $g \in \mathcal{G}$, 
    $a \in \mathcal{A}$, there exists a $G^{*}_{s:g,a} \in \mathbb{R}$ such that
    \[
    \bar{Q}^{*}_{M}(s,g,a) 
    = G^{*}_{s:g,a} + \bar{r}_{M}(s', g, a'), 
    \quad \text{where } s' \in \mathcal{G} 
    \text{ and } a' = \arg\max_{b \in \mathcal{A}} \bar{r}_{M}(s', g, b).
    \]
\end{restatable}

\smallskip
We assume the optimal Q-value functions for tasks $\universaltask$ and $\emptytask$, $\bar{Q}_{\mathcal{U}}^*$ and $\bar{Q}_{\varnothing}^*$, respectively, are known, and present an alternative way to obtain $\bar{Q}_M^*$ for any task $M$.

\begin{restatable}{definition}{definitionGoalsets}\label{def:goals_set}
    Let $M \in \mathcal{M}$ and define $\mathcal{G}^+_M = \{g \in \mathcal{G}: \bar{r}_{M}(g, g, a) = \universalreward \; \forall a \in \mathcal{A}\}$, i.e., $\mathcal{G}^+_M$ is the set of desired goals for task $M$. Denote $\mathcal{G}^+_i$ for the set of goals for task $M_i$.
\end{restatable}

\begin{restatable}{lemma}{llemaOptimal}\label{lemma:new_valuefunc}
(Proof in Appendix \ref{appendix_proofs})
Let $\mathcal{M}$ be the set of tasks defined in Equation \ref{def:task} and $\mathcal{G}^+_M$ like in Definition \ref{def:goals_set}. For any $g \in \mathcal{G}$ and $M \in \mathcal{M}$,
\[
\bar{Q}^*_M (s, g, a) =
\begin{cases}
    \bar{Q}_{\mathcal{U}}^*(s, g, a) \;\; \text{if $g \in \gplus_M$} &\forall (s, a) \in \mathcal{S} \times \mathcal{A},\\
    \bar{Q}_{\varnothing}^*(s, g, a) \;\; \text{if $g \in \mathcal{G} \setminus \gplus_M$} &\forall (s, a) \in \mathcal{S} \times \mathcal{A}.
\end{cases}
\]
In other words, if $g$ is a goal for task $M$, i.e. $g \in \gplus$, then the state-action slice for goal $g$ and task $M$ is that of the universal task for goal $g$. Similarly, if $g$ is \textbf{not} a goal for task $M$, i.e. $g \notin \gplus$, then it is the slice of the empty task.
\end{restatable}

Lemma \ref{lemma:new_valuefunc} states that for any $g$, the $\bar{Q}^*_M(\cdot, g, \cdot)$ slice is not arbitrary; it is either the $\bar{Q}^*_{\mathcal{U}}(\cdot, g, \cdot)$ slice or the $\bar{Q}^*_{\varnothing}(\cdot, g, \cdot)$ slice.
We can obtain the optimal extended Q-value function for any given task in the BTA by independently selecting the state-action slices corresponding to the universal task if the goal is active for task $M$ and to the empty task otherwise.
It then follows that it is not necessary to learn $2 + \lceil \log_2 (\mathcal{|G|}) \rceil$ base tasks to obtain composition because we would be learning duplicate information.
Figure \ref{fig:redundant_learning} portrays this idea by conceptually representing the optimal value functions as a 3D tensor for two tasks, $M_1$ and $M_2$, that pursue goals $\gplus_{M_1} = \{g_2, g_3\}$ and  $\gplus_{M_2  } = \{g_1, g_2\}$, respectively.
In particular, both tasks share a goal $g_2$ and hence the optimal state-action slice corresponding to this goal coincides. Moreover, the slices for shared non-goals, i.e. elements of $\mathcal{G}$ that neither $M_1$ nor $M_2$ pursue, like $g_4$ are also shared across tasks.

\begin{figure}
    \centering
    \includegraphics[width=0.55\linewidth]{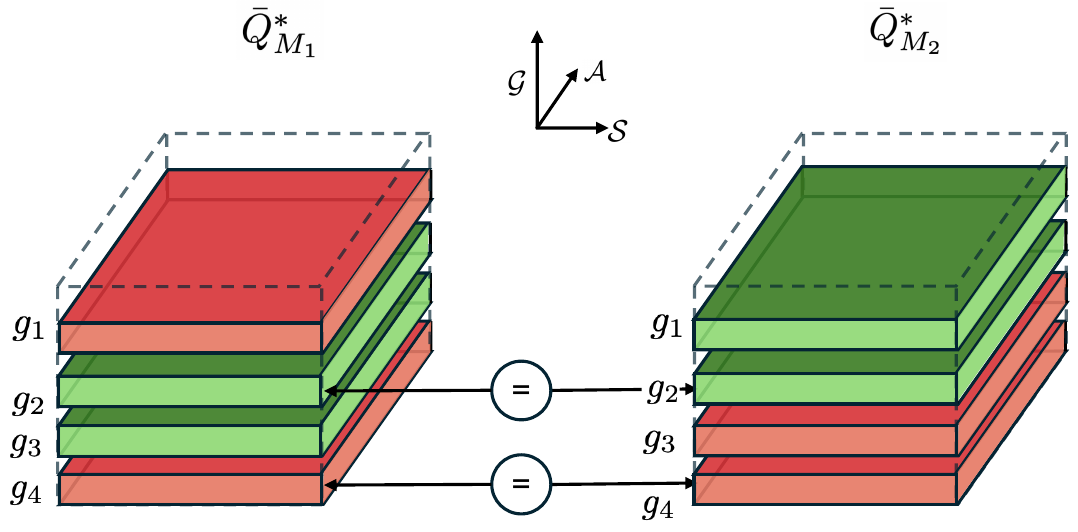}
    \caption{
        Optimal extended Q-value functions for two tasks $M_1$ and $M_2$ that pursue goals $\{g_2, g_3\}$ and $\{g_1, g_2\}$, respectively.
        Pursued goal slices are colored green and non-pursued red.
    }
    \label{fig:redundant_learning}
\end{figure}

It is important to note that we decide to learn the universal and empty tasks because they constitute an intuitive separation of the goal space but we can learn any two tasks $M_1, M_2 \in \mathcal{M}$ such that $\gplus_1 \uplus \gplus_2 = \mathcal{G}$\footnote{Here $\uplus$ denotes disjoint union.}, which means we learn a state-action slice for pursuing and not pursuing every goal.
In summary, any combination of tasks can be chosen as long as the agent learns the expected return of reaching the goal when it is desired and undesired in the task.

Under this new formulation, composition becomes straightforward since it suffices to obtain the set of desired goals for the composed task and construct the extended Q-value function using Lemma~\ref{lemma:new_valuefunc}.
We formalize this notion by defining, in Lemma \ref{lemma:homomorphism}, an isomorphism between the set of tasks, $\mathcal{M}$, and the power set of goals, $2^{\mathcal{G}}$, allowing us to express task composition via goal set composition in Theorem \ref{thm:composition_using_selection}.

\begin{restatable}{lemma}{llemaIsomorphism}\label{lemma:homomorphism}
    \textit{(Proof in Appendix \ref{appendix_proofs})}
    Let $\mathcal{F}: \mathcal{M} \to 2^{\mathcal{G}}$ be the map from $\mathcal{M}$ to $2^{\mathcal{G}}$ such that $\mathcal{F}(M) = \mathcal{G}^+_M$.
    Then $\mathcal{F}$ is an isomorphism between the task Boolean algebra
    $(\mathcal{M}, \lor, \land, \neg, \mathcal{M}_\mathcal{U}, \mathcal{M}_\varnothing)$
    and the power set algebra for goals $(2^{\mathcal{G}}, \cup, \cap, \: ', \mathcal{G}, \varnothing)$.
\end{restatable}

\begin{restatable}{theorem}{thmOne}
\label{thm:composition_using_selection}
    \textit{(Proof in Appendix \ref{appendix_proofs})}
    Let $M_1, M_2 \in \mathcal{M}$ be two tasks, and $\gplus_{M_1}, \gplus_{M_2}$ the respective sets of desired goals.
    Then,
    \begin{enumerate}
        \item[(i)]
        $
        \;\; \bar{Q}^*_{\neg M_1} (s, g, a) =
        \begin{cases}
        \bar{Q}^*_{\mathcal{\mathcal{U}}} (s, g, a)
        \;\; \forall (s, a)
        \quad \text{if $g \in \mathcal{G} \setminus \gplus_{M_1}$}, \\
        \bar{Q}^*_{\mathcal{\varnothing}} (s, g, a)
        \;\; \forall (s, a)
        \quad \text{if $g \in \gplus_{M_1}$}.
        \end{cases}
        $
        
        \item[(ii)]
        $
        \;\; \bar{Q}^*_{M_1 \lor M_2} (s, g, a) =
        \begin{cases}
        \bar{Q}^*_{\mathcal{\mathcal{U}}} (s, g, a)
        \;\; \forall (s, a)
        \quad \text{if $g \in \gplus_{M_1} \cup \gplus_{M_2}$}, \\
        \bar{Q}^*_{\mathcal{\varnothing}} (s, g, a)
        \;\; \forall (s, a)
        \quad \text{if $g \in \mathcal{G} \setminus (\gplus_{M_1} \cup \gplus_{M_2})$}.
        \end{cases}
        $

         \item[(iii)]
        $
        \;\; \bar{Q}^*_{M_1 \land M_2} (s, g, a) =
        \begin{cases}
        \bar{Q}^*_{\mathcal{\mathcal{U}}} (s, g, a)
        \;\; \forall (s, a)
        \quad \text{if $g \in \gplus_{M_1} \cap \gplus_{M_2}$}, \\
        \bar{Q}^*_{\mathcal{\varnothing}} (s, g, a)
        \;\; \forall (s, a)
        \quad \text{if $g \in \mathcal{G} \setminus (\gplus_{M_1} \cap \gplus_{M_2})$}.
        \end{cases}
        $
    \end{enumerate}
\end{restatable}

Theorem \ref{thm:composition_using_selection} presents an alternative way to obtain task composition, characterized by the set of desired goals for each task.
Furthermore it brings a series of advantages with respect to the original method of composition presented in the paper.

\paragraph{Computational cost.} The formulation presented in Theorem~\ref{thm:composition_using_selection} offers computational advantages over the original BTA composition method, both in learning and composition.
The original BTA framework requires learning $\lceil \log_2 |\mathcal{G}| \rceil$ base tasks, in addition to $\universaltask$ and $\emptytask$, whereas in our new formulation it suffices to learn only two tasks: $\universaltask$ and $\emptytask$.
This reduces the number of parameters required for \textit{learning} from $\mathcal{O}\left(\lceil \log_2 |\mathcal{G}| \rceil \times |\mathcal{G}| \times |\mathcal{S}| \times |\mathcal{A}|\right)$ to $\mathcal{O}\left(|\mathcal{G}| \times |\mathcal{S}| \times |\mathcal{A}|\right)$.
The \textit{cost of composition} is also reduced: instead of performing a sequence of $\lceil \log_2 |\mathcal{G}| \rceil$ element-wise operations, with worst-case cost $\mathcal{O}\left(\lceil \log_2 |\mathcal{G}| \rceil \times |\mathcal{G}| \times |\mathcal{S}| \times |\mathcal{A}|\right)$, it performs set operations on the desired goals and memory access/copying of the corresponding per-goal state-action slices, with cost $\mathcal{O}\left(|\mathcal{G}| \times |\mathcal{S}| \times |\mathcal{A}|\right)$.
Corollary~\ref{cor:q-value-empty-less-than-universal} provides the key insight, i.e., leveraging the order between task slices, the computations over the state-action dimensions can be done in $\mathcal{O}(1)$ time by simply choosing the bigger slice.

\begin{restatable}{corollary}{corollaryOne} \textit{(Proof in Appendix \ref{appendix_proofs})} \label{cor:q-value-empty-less-than-universal}
    The universal slice is uniformly greater than the empty slice for all goals, states and actions.
    $
    \bar{Q}^*_{\varnothing} (s, g, a) \le \bar{Q}^*_{\mathcal{U}} (s, g, a)
    \;\; \forall (s, g, a) \in \mathcal{S} \times \mathcal{G} \times \mathcal{A}.
    $
\end{restatable}

\paragraph{What happens for stochastic MDPs?}
This work addresses the matter of how many value functions are needed to construct every optimal extended Q-value function in \emph{deterministic} MDPs.
We can also wonder how many value functions would be needed to construct every optimal extended Q-value function in \emph{stochastic} MDPs.
In the former case we know that whatever the goal set, at most we need the functions under $|\mathcal{G}|$ different policies to obtain the optimal policy through composition.
For a task with a given set of desired goals we can operate independently learned slices for each set of goals.
However, this does not hold in a stochastic MDP, because this independence between goals inside the desired goal set disappears.
Proposition \ref{prop:counterex} proves these points by showing that the nature of optimal policies in the stochastic setting is exponential with respect to the number of total goals.

\begin{restatable}{proposition}{propCounterex}\label{prop:counterex}(Proof in Appendix \ref{appendix_proofs}.)
There exist MDPs with stochastic transition dynamics and goal set $\mathcal{G}$ for which there will be a different optimal action for each of the $2^{|\mathcal{G}|}-1$ desired goal sets, $\mathcal{G}^+ \subseteq \mathcal{G}$, with $|\mathcal{G}^+| \ge 1$.
\end{restatable}

Consequently, we argue that there can be no strategy for the stochastic case that, in each state, takes a maximum (or another type of convex combination) over Q-value functions of a `small' set of policies (linear or even polynomial in $|\mathcal{G}|$).
In the worst case, a different action is optimal for each possible goal set (of which there can be exponentially many), meaning that a convex combination would need to be taken over an exponential number of `primitive' Q-value functions\footnote{Note that there are special cases where such constructions are possible, see e.g. \cite{infante2022globally}.}.
In fact, \cite{tasse2022generalisation} extend the BTA framework to stochastic MDPs, but optimality of composed value functions no longer holds, and only approximate guarantees can be established.

\section{Experiments}
\label{section:experiments}
Thus far, we have formalized and proposed a theoretical improvement in both the learning and composition costs of the BTA.
The aim of the experiments presented here is to compare our proposed method to that of the original BTA \citep{tasse2020bta} in order to confirm these improvements empirically.
Additionally, we compare the advantage of this new method with respect to that of Skill Machines \citep{tasse2024skillmachinestemporallogic}, where we aim to show improvements only in composition cost, since this framework already takes advantage of training improvements by training two extended value functions (Algorithm 1 in \citep{tasse2024skillmachinestemporallogic}).
In particular we study these methods in \textit{sub-optimal scenarios} with an increasing number of training iterations and \textit{upon convergence}.

\subsection{Experimental setup}
\begin{wrapfigure}[10]{r}{0.25\textwidth}
    \vspace{-1.2em}
    \centering
    \includegraphics[width=\linewidth]{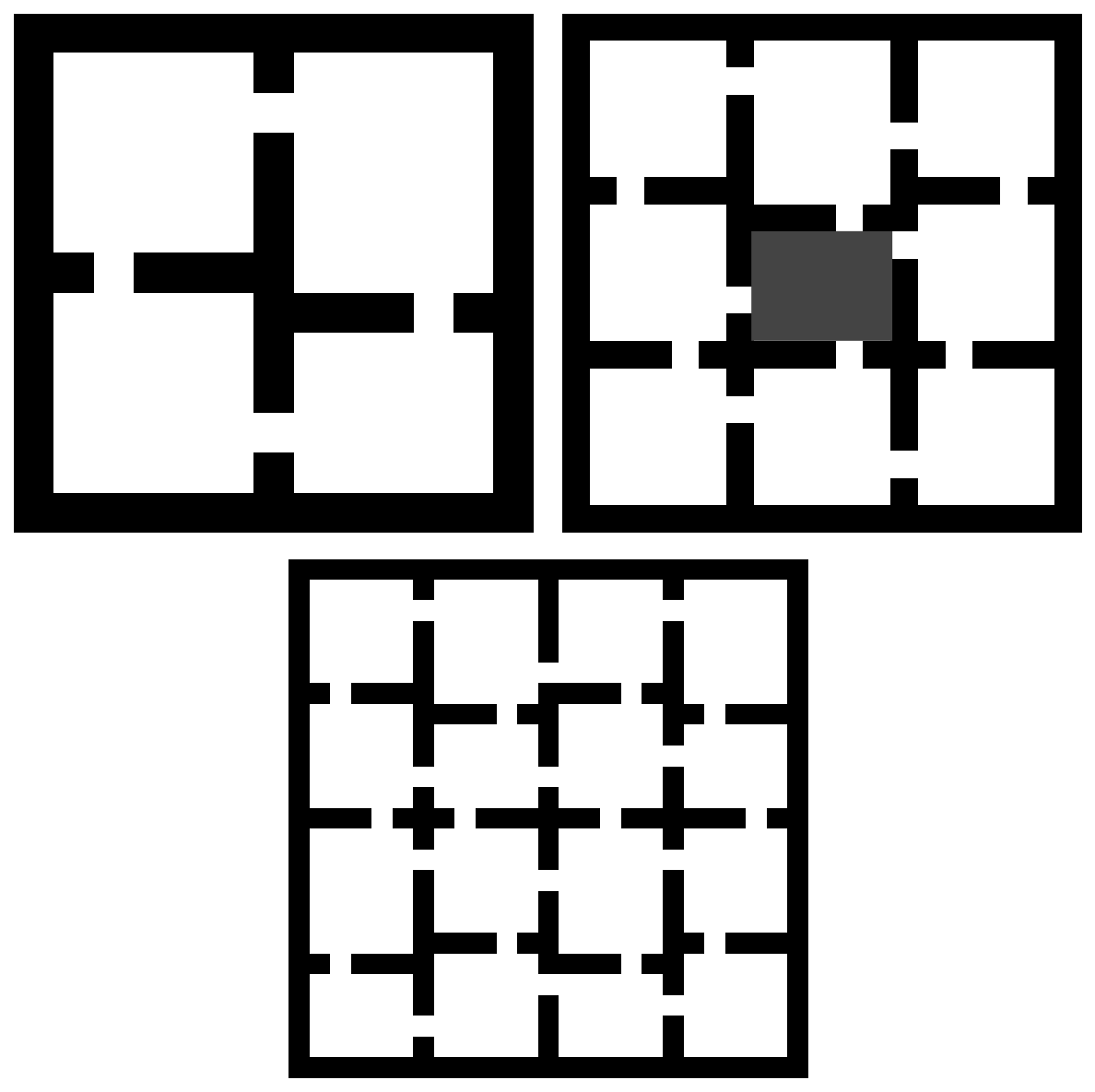}
    \label{fig:composition_time_compared}
\end{wrapfigure}
\paragraph{Rooms environment \citep{Sutton1998}.}
We extend the four rooms environment \citep{Sutton1998} to incorporate an exponentially growing number of rooms so we can place one goal per room and study the scaling law of both methods.
Consider a grid-world in which the agent navigates between rooms separated by walls with narrow doorways. The state space consists of all navigable grid cells and actions are the four cardinal movement directions plus a stay action.
We instantiate three variants: a $2\times2$ rooms layout with $|\mathcal{G}|=4$ goals (one per room), a $3\times3$ rooms layout with $|\mathcal{G}|=8$ goals (the middle room contains no goals), and a $4\times4$ rooms layout with $|\mathcal{G}|=16$ goals. We set $r_\mathcal{U}=2$, $r_\varnothing=-0.1$, a step reward of $-0.1$, and a discount factor of $1$.
Given the exponentially growing number of tasks, we randomly sample at most $5$ tasks for each goal set length for diversity and train the universal, empty, and base tasks, as well as all tasks, separately to obtain their optimal policies.
All tasks are trained with the same number of maximum iterations, meaning that the base task method uses more compute since it requires learning more tasks.

\begin{wrapfigure}[9]{l}{0.24\textwidth}
    \vspace{-1.2em}
    \centering
    \includegraphics[width=\linewidth]{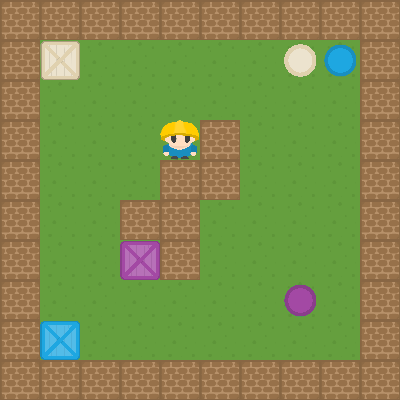}
\end{wrapfigure}
\paragraph{Boxman environment \citep{tasse2020bta}.}
A high-dimensional object collection environment with pixel-based observations of size $84 \times 84 \times 3$ (RGB).
The agent must collect objects defined by combinations of colors (beige, blue and purple) and shapes (squares and circles), which constitute skill primitives.
Tasks are learned using deep Q-learning with a CNN+MLP, where we train for 3 full runs and save intermediate checkpoints. We then take each checkpoint and evaluate in the environment, averaging over runs.
This allows us to understand how the learned policies converge under different methods.

\begin{wrapfigure}[7]{r}{0.24\textwidth}
    \vspace{-1.5em}
    \centering
    \includegraphics[width=\linewidth]{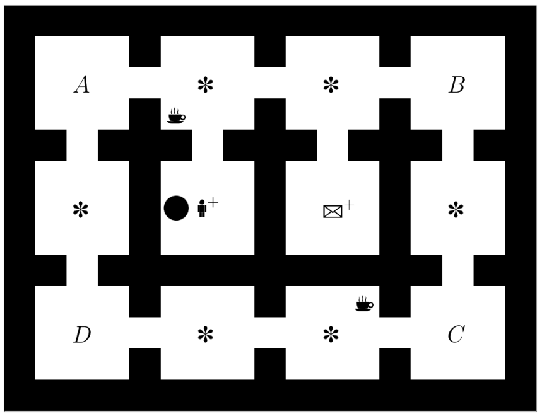}
\end{wrapfigure}
\paragraph{Office Gridworld \citep{Toro_Icarte_2022rewardmachines} for Skill Machines.}
A tabular grid-world containing labelled locations: rooms $A, B, C, D$, a coffee machine, a mail room, an office, and decoration cells. States are grid cells, actions are the four cardinal directions, and the goal set $\mathcal{G}$ consists of the labelled locations. Tasks are temporal logic specifications over these propositions, such as delivering coffee to the office before picking up mail or avoiding decoration cells.

\begin{wrapfigure}[8]{l}{0.22\textwidth}
    \vspace{-1.5em}
    \centering
    \includegraphics[width=\linewidth]{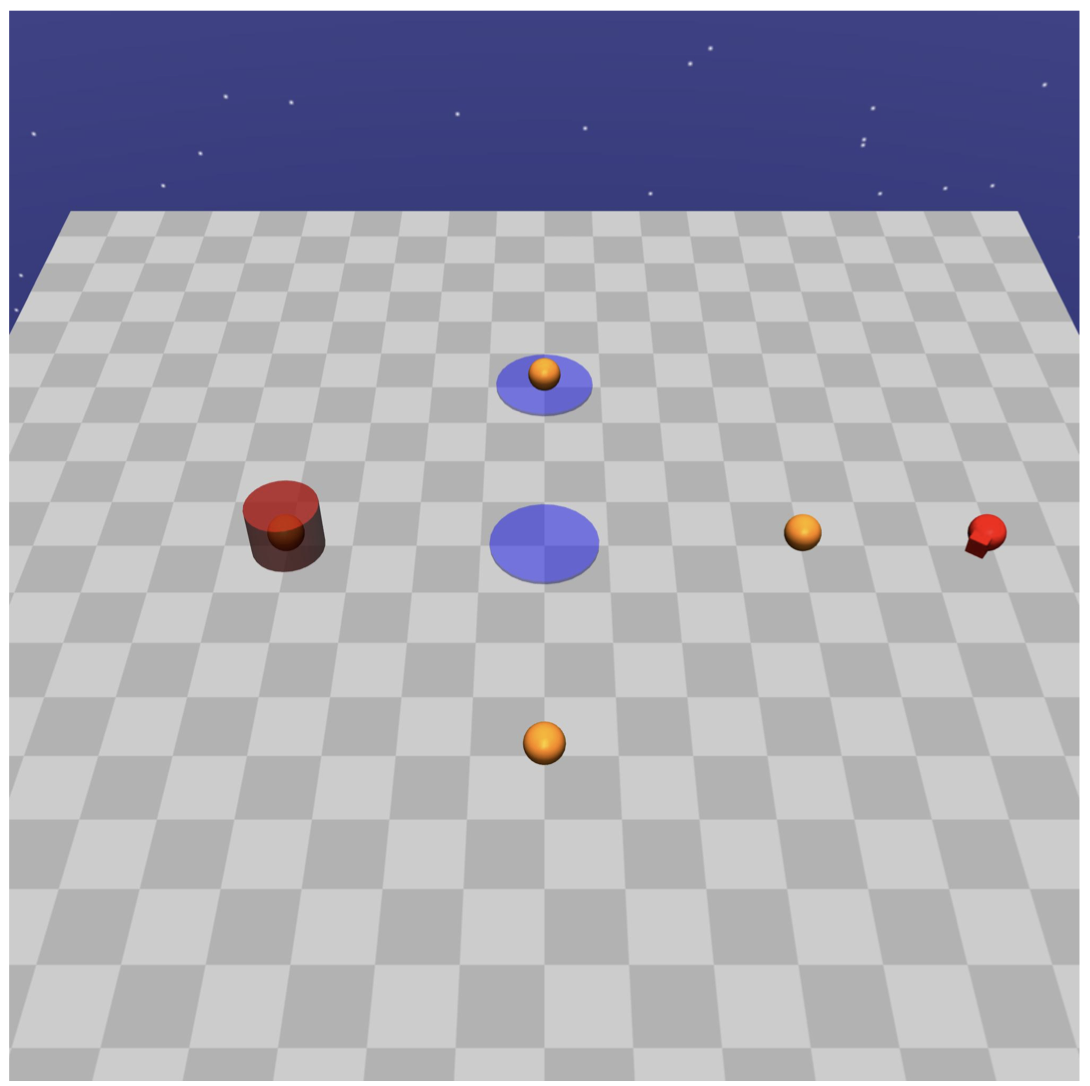}
\end{wrapfigure}
\textbf{Safety Gym Domain \citep{Achiam2019BenchmarkingSE} for Skill Machines.}
A continuous control environment in which a point-mass robot navigates to various regions using a continuous state space $\mathcal{S} = \mathbb{R}^{60}$ (comprising joint velocities, LiDAR readings, and other sensory inputs) and a continuous action space $\mathcal{A} = \mathbb{R}^{2}$ (direction and force). Tasks are defined over $|\mathcal{P}| = 3$ propositions corresponding to a red cylinder, green buttons, and blue regions, with a single constraint $\mathcal{C} = \{c\}$ tracking blue-region violations, yielding $|\mathcal{P} \cup \mathcal{C}| = 4$ skill primitives trained using TD3.

In both of these environments we measure successes in the LTL specification as well as total return.
We evaluate \textit{three methods} in this domain: the \textit{original} Skill Machines approach, which learns $2$ value functions and performs composition via Boolean operations (ultimately these are min, max, sum);
the goal-set based approach, which learns $2$ value functions and composes via slice selection; and the base tasks method, which learns more than $2$ tasks and composes via Boolean operations.

\subsection{Results.}
\paragraph{Rooms environment.}
Figure \ref{fig:convergence} shows the evolution of the average returns of the sampled tasks as we increase the number of training iterations.
We observe that returns given by the goal-set method are consistently higher than the base task approach.
With sufficient iterations, both methods learn the optimal policy, but composition through the universal and empty tasks converges to the optimal policy faster.
We hypothesize this behavior is due to the mixing of the value functions obtained via the $\max$ and $\min$ operations being detrimental.
In this sense, Corollary \ref{cor:q-value-empty-less-than-universal} is not guaranteed to hold until the Q-value functions converge, thus these operations might select an arbitrary number of parameters from one slice and the rest from the other.
In turn, full slice selection still maintains the relative ordering of slices, which we show is a fundamental property of the optimal value function.

\begin{figure}
    \centering
    \includegraphics[width=\linewidth]{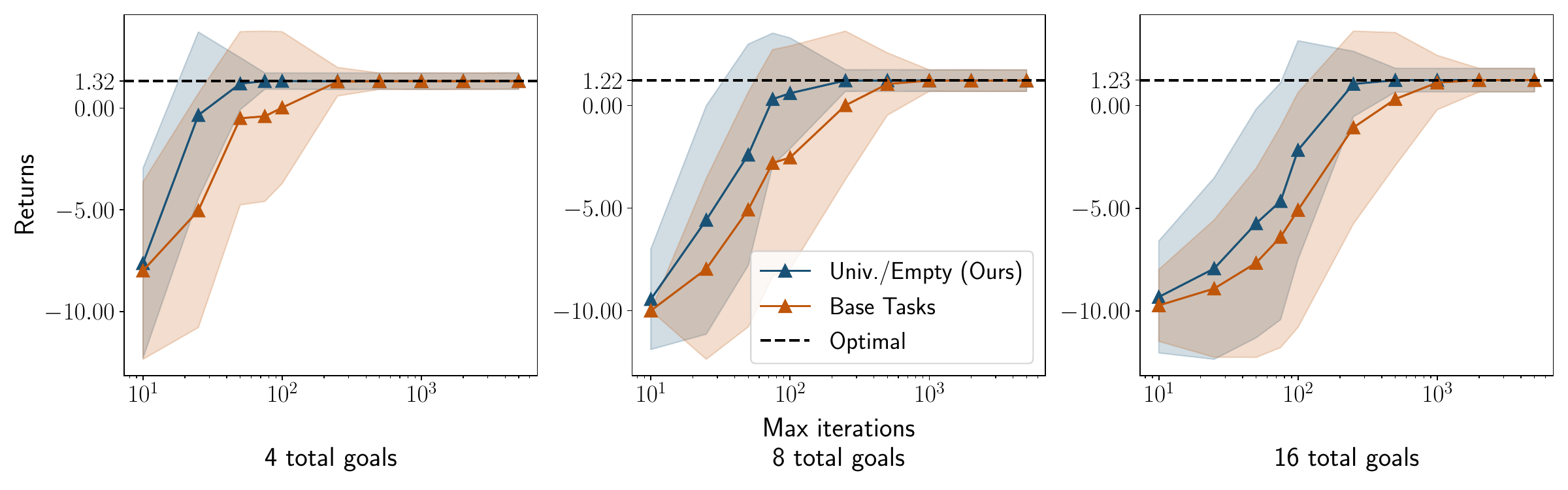}
    \caption{
        Evolution of the mean ($\pm$ std) return for each goal set length with the number of maximum iterations of the training algorithm.
        The returns for each task are measured on $1000$ episodes and averaged over $5$ tasks.
        The number of training iterations (x-axis) is per-value function.
    }
    \label{fig:convergence}
\end{figure}

\begin{wrapfigure}[11]{r}{0.45\textwidth}
    \vspace{-1em}
    \centering
    \vspace{-0.5em}
    \includegraphics[width=\linewidth]{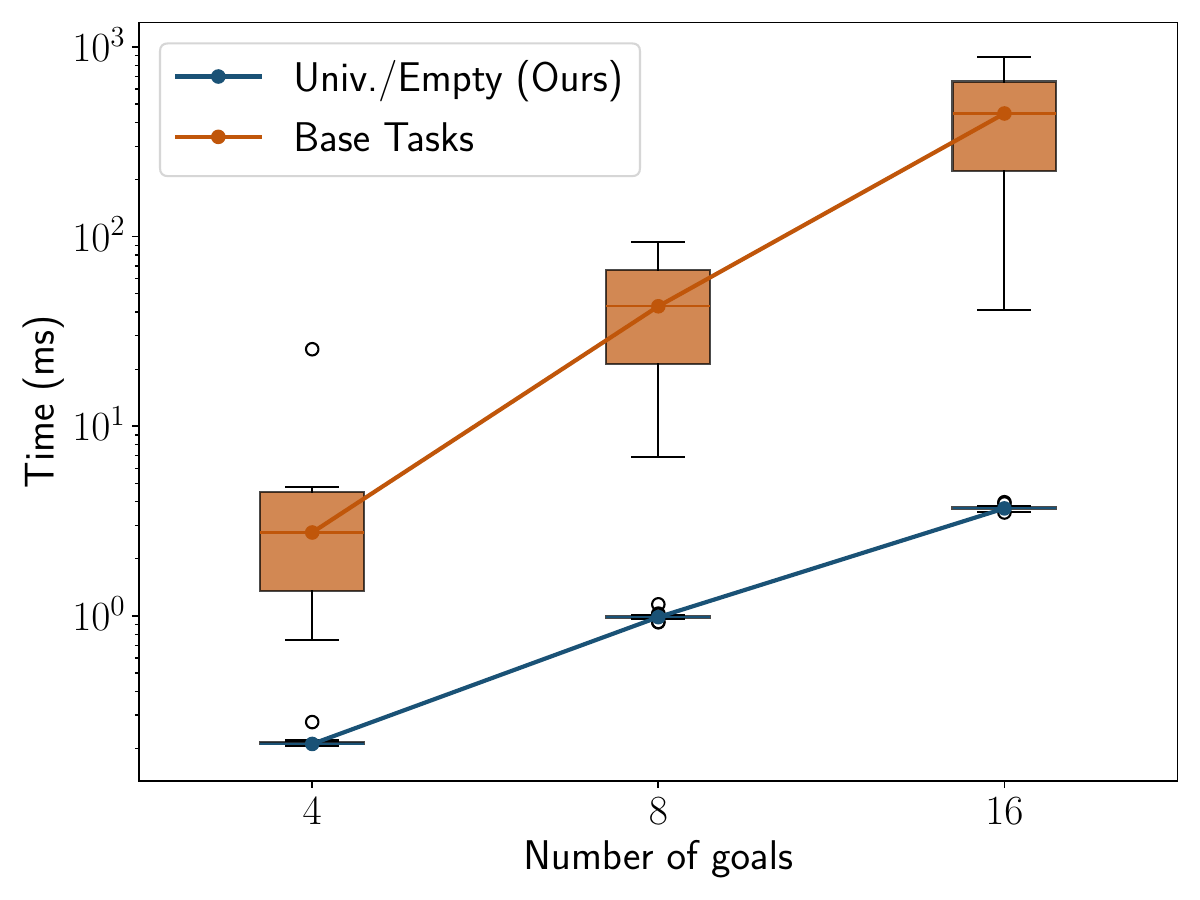}
    \label{fig:composition_time_compared}
\end{wrapfigure}
\textit{Temporal cost of composition.}
We measure the time taken for performing composition across all tasks using base tasks versus using only the universal and empty tasks in the goal-set based formulation.
The right figure shows that \textsc{Univ./Empty} consistently achieves lower composition times and exhibits substantially less variability than the \textsc{Base Tasks} approach across all problem sizes.
Moreover, the gap between the two methods widens as the number of goals increases. 
Note that the y-axis is in logarithmic scale.

\paragraph{Boxman environment.}
Figure \ref{fig:boxman_convergence} shows convergence of both methods in the Boxman 
environment, where we include two representations of the same training runs. 
Figure \ref{fig:boxman_sub1} shows the evolution of the average return as a function of the number of training iterations per UVFA. We observe that both methods reach a similar final return, but the original base task method converges to the same optimal policy with fewer iterations per value function.
However, this advantage disappears when we correct for the total number of training iterations per method. Figure \ref{fig:boxman_sub2} shows the same results as a function of the total number of training iterations summed across all UVFAs.
Since the base task method requires learning more value functions, its total training cost is higher, and it remains at a suboptimal return for a significantly longer total budget, whereas our proposed method reaches near-optimal performance earlier.
The figures also confirm that in a function approximation setting, learning additional base tasks does not translate to improved performance upon convergence.

\begin{figure}[htbp]
    \centering
    \begin{subfigure}[b]{0.45\textwidth}
        \centering
        \includegraphics[width=\textwidth]{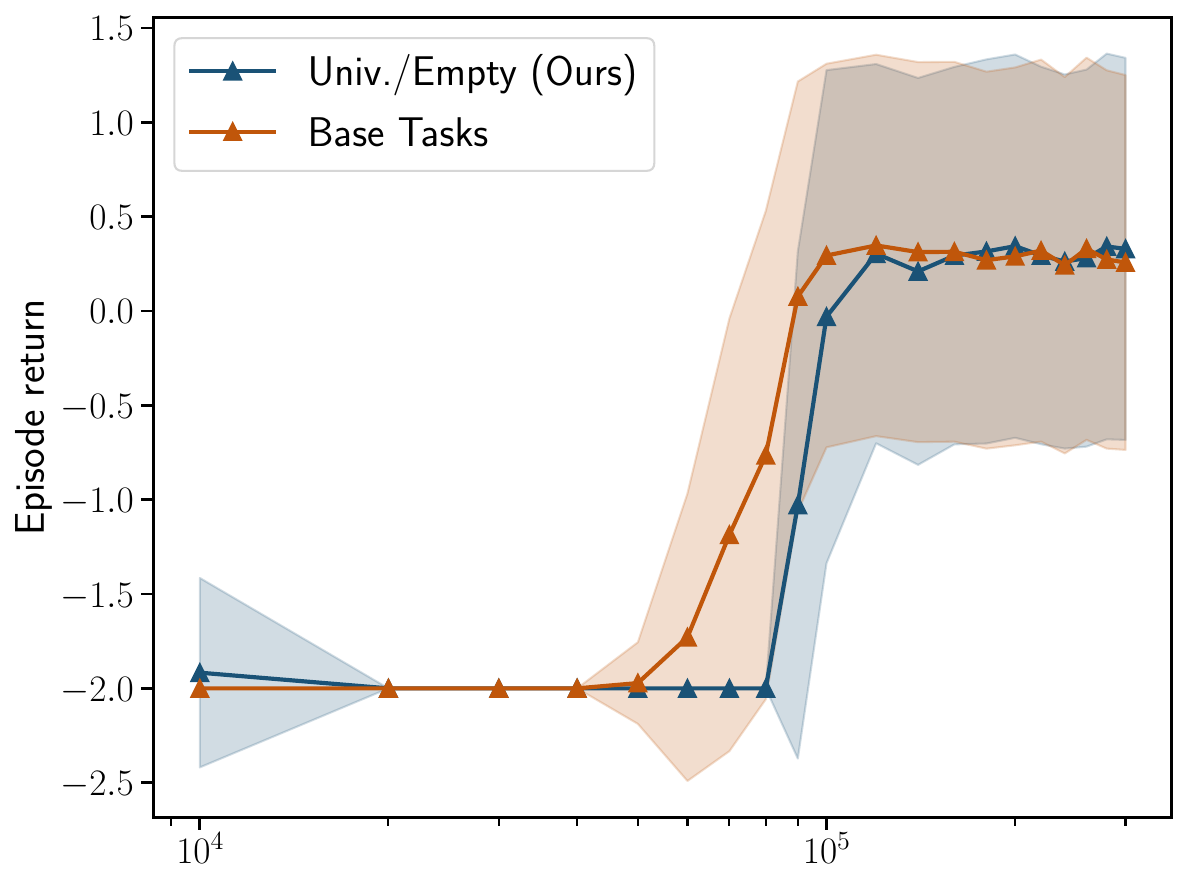}
        \caption{Comparison for a fixed per-UVFA budget.}
        \label{fig:boxman_sub1}
    \end{subfigure}
    \hfill
    \begin{subfigure}[b]{0.45\textwidth}
        \centering
        \includegraphics[width=\textwidth]{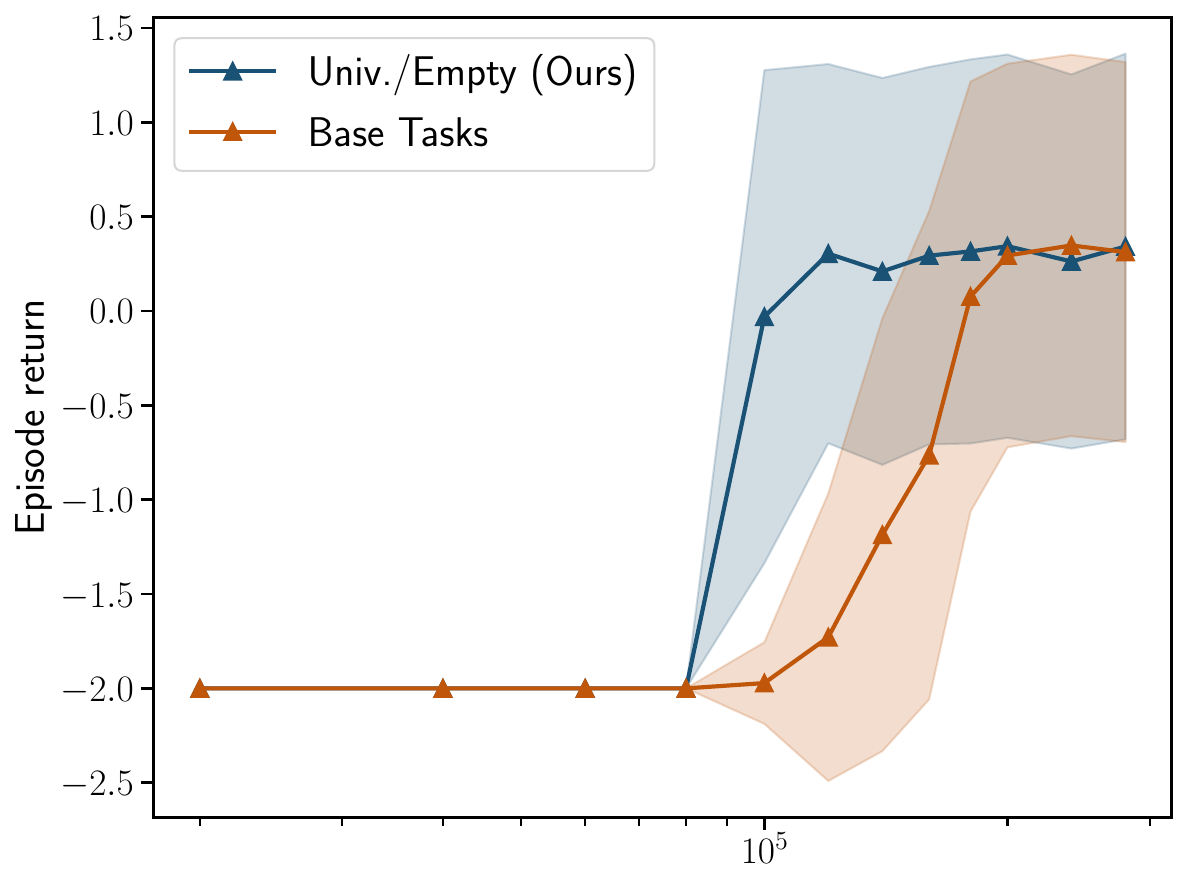}
        \caption{Comparison for a fixed total training budget.}
        \label{fig:boxman_sub2}
    \end{subfigure}
    
    \caption{Evolution of the average reward ($\pm$ std) across all tasks in the Boxman environment. The x-axis represents the number of training iterations.
    Results are averaged over tasks and $3$ runs.
    }
    \label{fig:boxman_convergence}
\end{figure}

\paragraph{Skill machines environments.}
Figure \ref{fig:office:returns_successes} depicts the success rate and average episode return of the LTL specification for the average of the different tasks in the office environment.
We see that all methods converge almost in an equivalent manner, confirming that extra efforts to train base tasks yield no additional returns.
Figure \ref{fig:office:times} confirms that the goal-set based approach (Univ./Empty) has a lower composition cost, since it consists only of slice selection.
We further observe that the original and base task methods have similar cost given that their method for composition is the same.
See Appendix \ref{appendix:b:detailed-experimental-design} and \ref{appendix:b:office} for task explanations and per-task metrics.
The results for the Safety Gym environment further confirm these conclusions and can be found in Appendix~\ref{appendix:b:safety}.

\begin{figure}[htbp]
    \centering
    \begin{subfigure}[b]{0.58\textwidth}
        \centering
        \includegraphics[width=\textwidth]{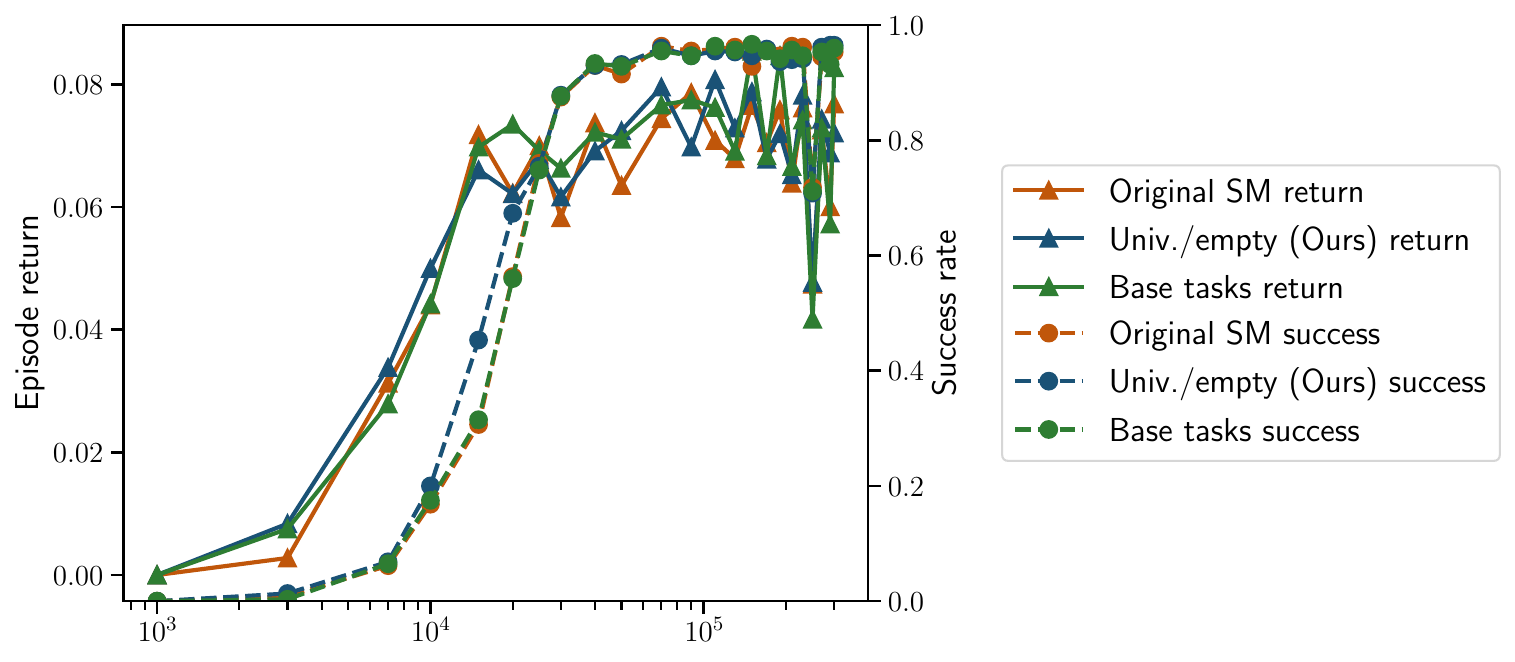}
        \caption{Average returns (left axis) and success rate (right axis).}
        \label{fig:office:returns_successes}
    \end{subfigure}
    \hfill 
    \begin{subfigure}[b]{0.4\textwidth}
        \centering
        \includegraphics[width=\textwidth]{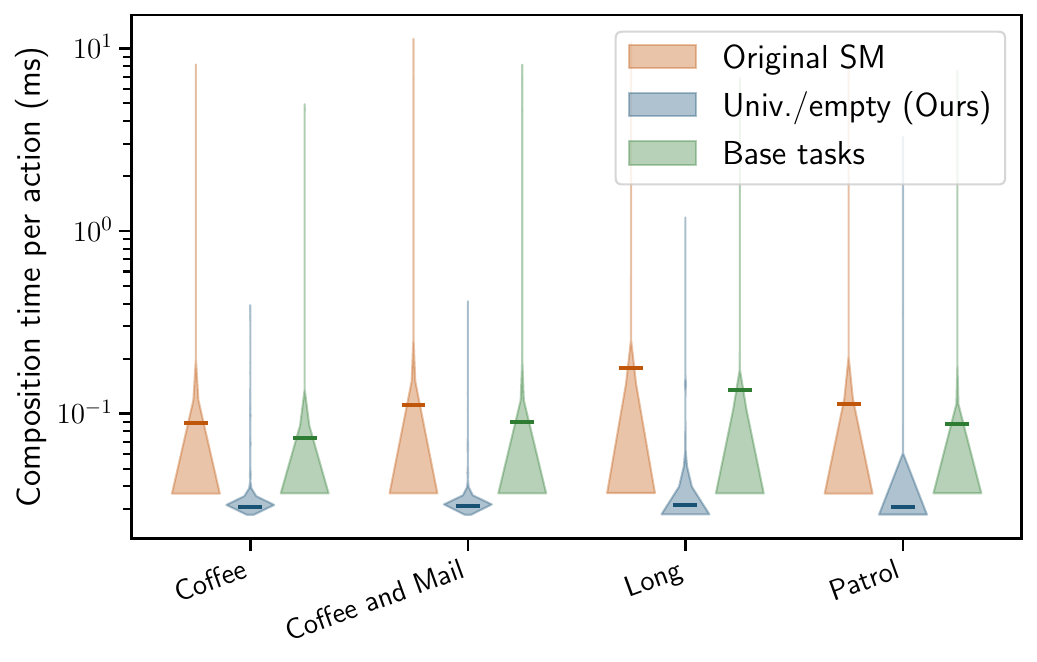}
        \caption{Violin plot of composition time (log sc.).}
        \label{fig:office:times}
    \end{subfigure}
    
    \caption{Evolution of average ($\pm$ std) successes and returns across $3$ runs and $4$ tasks in the temporal composition of actions for the Office Gridworld environment.
    }
    \label{fig:skill_machines_gridworld}
\end{figure}

\begingroup
\paragraph{Rooms environment with stochastic transitions.}
Our theoretical results assume deterministic dynamics, which limits the settings in which optimal composition is guaranteed.
We therefore investigate experimentally whether the proposed method remains useful when transitions are stochastic.
We vary the slip probability $p$ from $0$ to $0.5$ to evaluate both composition methods when the deterministic transition assumption is relaxed.
With probability $p$, the intended movement is replaced by a uniformly sampled alternative direction, while the stay action is unchanged.
For each value of $p$, we retrain the value functions and evaluate the composed policies without further learning.
Figure~\ref{fig:slip_returns} shows the results for $|\mathcal{G}|=16$, with the results for $|\mathcal{G}|\in\{4,8\}$ in Figure~\ref{fig:slip_returns_small} in Appendix~\ref{app:stochastic-rooms}.
We observe that both composition methods obtain similar returns at low slip probabilities.
As slip increases, returns decrease for both methods, but the original BTA declines more sharply and the gap in favour of goal-set composition widens.
The dotted curve shows the optimal value function $V^*$ averaged over the evaluated tasks and initial states, providing a reference for the returns achievable at each slip probability.
Both composition methods remain close to this reference at low slip probabilities, but fall below it as slip increases, with goal-set composition retaining a smaller gap to optimal returns.

We hypothesize that this difference is due to approximation errors being propagated through successive composition operations.
The base task method learns $\lceil\log_2|\mathcal{G}|\rceil$ basis tasks, but obtaining a desired task requires composing their value functions according to the Boolean expression.
Errors in intermediate value functions are therefore passed to subsequent operations.
Negation is particularly problematic because it involves a reflection using three approximate value functions, corresponding to the input, universal and empty tasks, and does not preserve the original input values.
In comparison, $\min$ and $\max$ always return values from their original inputs.
Our method performs the logical operations on goal sets and selects the corresponding slices from only two learned value functions, avoiding successive composition of approximate value functions.
However, this advantage disappears in the Boxman environment with function approximation (Figure~\ref{fig:boxman_slip_returns} in Appendix~\ref{app:boxman-results}), a phenomenon we attribute to the simple tasks of this environment, as compared to the four, eight and sixteen rooms environment, which involve tasks with longer sequences of Boolean operations.
The full experimental setup and further discussion are provided in Appendix \ref{app:stochastic-rooms}.

\begin{figure}[htbp]
            \centering
    \includegraphics[width=0.5\linewidth]{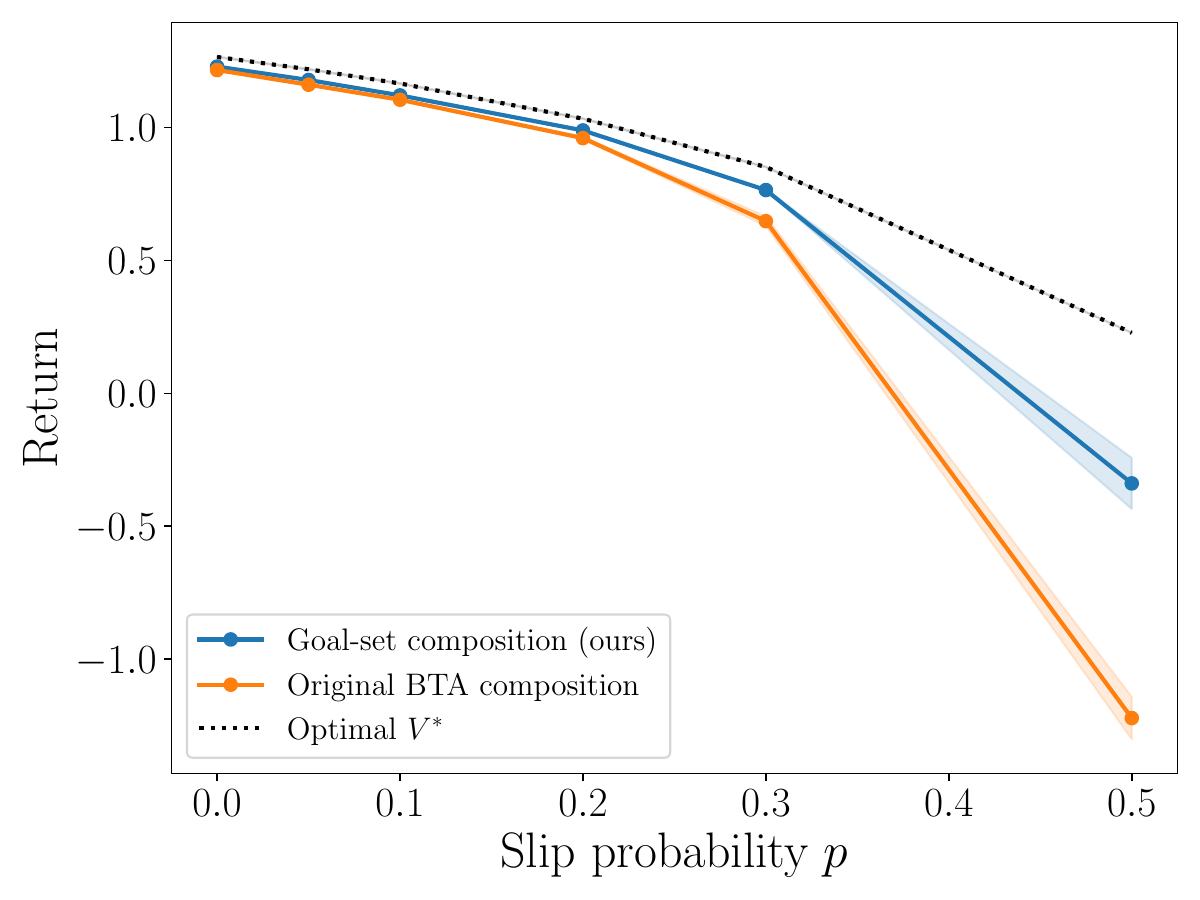}
    \caption{Average return as a function of slip probability in the Rooms environment with $|\mathcal{G}|=16$.
    Solid lines show the returns of policies composed using the goal-set method and base tasks (original BTA).
    The dotted black line shows the optimal value function $V^*$ averaged over the evaluated tasks and initial states.
    Composition results are averaged over $3$ seeds.
    }
    \label{fig:slip_returns}
\end{figure}

\par
\endgroup

\section{Conclusion}
In this work we revisited the Boolean Task Algebra and studied structural properties of its environmental assumptions that lead to a simplified formulation of the optimal value functions, see Section \ref{section:extension_bta}.
First, we proved that the expressive space of extended optimal Q-value functions is fully determined by the universal and empty tasks, thereby confirming our first contribution in Lemma \ref{lemma:new_valuefunc}, that is, the effective collapse of the representation space to two fundamental value functions.
Building on this insight, we introduced an isomorphism in Lemma \ref{lemma:homomorphism} between tasks and subsets of goals, enabling a goal-set-based characterization of composition.
This establishes our second contribution, provided in Theorem \ref{thm:composition_using_selection}, as it allows logical operations to be performed directly on goal sets rather than through element-wise manipulation of extended Q-value functions.
We further analyze whether our proposed method would extend to stochastic MDPs and propose a counterexample in Proposition \ref{prop:counterex} where we show that the number of optimal policies to account for in order to achieve optimal composition is exponential in the number of total goals.

Finally, Section \ref{section:experiments} shows that our goal-set method reduces composition time across tabular, function approximation, visual and continuous control environments, for both standard and LTL task specifications. This includes training and composition time improvements for standard BTA, and composition-only gains for Skill Machines.

\bibliography{refs}

\newpage
\appendix
\section{Proofs}\label{appendix_proofs}

In this section we present proofs for the different mathematical statements.
We begin by restating the following result and definition.

\propSharedG*

\definitionGoalsets*

\subsection{Characterizing optimal extended value functions}
\llemaOptimal*

\begin{proof}\label{proof:lemma:new_valuefunc} \textit{(Lemma \ref{lemma:new_valuefunc})}
Let $M \in \mathcal{M}$, $g \in \mathcal{G}$, $(s, a) \in \mathcal{S} \times \mathcal{A}$.
If $g$ is unreachable from $s$, from the proof of Lemma 2 from \cite{tasse2020bta}, we know that $M_g = \mathcal{M}_{\mathcal{U}, g}$, where $M_g$ and $\mathcal{M}_{\mathcal{U}, g}$ are MDPs with the same transition dynamics as $M$ and $\mathcal{M}_{\mathcal{U}}$ and with reward functions $r_{M_g} (s, a) = \bar{r}_{M} (s, g, a)$ and $r_{\mathcal{M}_{\mathcal{U}, g}} = \bar{r}_{\universaltask} (s, g, a)$, respectively.
By uniqueness of the solution to the Bellman optimality equation, their optimal Q-functions must be identical.

We now study the case where $g$ is reachable from state $s$.
From Corollary 1 we know that for these $s, g, a, \; \exists \!\; G_{s:g,a}^* \in \R$ such that $\bar{Q}_{M}^*(s,g,a) = G_{s:g,a}^* + \bar{r}_{M}(s',g,a')$, where $s' \in \mathcal{G}$ and $a' = \underset{b \in \mathcal{A}}{\arg\max} \, \bar{r}_M (s', g, b)$. Furthermore, from the proof of Lemma 2 we know that $s' = g$, hence $\bar{Q}_{M}^*(s,g,a) = G_{s:g,a}^* + \bar{r}_{M}(g,g,a')$.
Note that $G_{s:g,a}^*$ is independent of the task.
We now proceed by cases.

\noindent
$\smallsquare$ If $g$ is a goal for task $M$, we have that by definition of $\bar{r}_{M}$ and $\bar{r}_{\universaltask}$, $\bar{r}_{M}(g,g,a') = r_{M}(g,a')$ and $\bar{r}_{\universaltask}(g,g,a') = r_{\universaltask}(g,a')$.
From the definition of  $\mathcal{M}$, $r_{M}(g,a'), r_{\universaltask}(g,a') \in \{\emptyreward, \universalreward\}$, and more concretely $r_{M}(g,a') = r_{\universaltask}(g,a') = \universalreward$ given that $g$ is a goal for both $M$ and $\universaltask$.
Consequently, we have that $\bar{r}_{M}(g,g,a') = \bar{r}_{\universaltask}(g,g,a')$ and together with Corollary 1 as well as the fact that $g$ and $a'$ are arbitrary, we obtain
\begin{align}
    \bar{Q}_{M}^*(s,g,a)
    = G_{s:g,a}^* + \bar{r}_{M}(g,g,a')
    = G_{s:g,a}^* + \bar{r}_{\universaltask}(g,g,a')
    = \bar{Q}_{\mathcal{U}}^*(s,g,a).
\end{align}

\noindent
$\smallsquare$ If $g$ is not a goal for task $M$, the reasoning is analogous,
$\bar{r}_{M}(g,g,a') = r_{M}(g,a') = \emptyreward = r_{\emptytask}(g,a') = \bar{r}_{\emptytask}(g,g,a')$. Finally,
\begin{align}
    \bar{Q}_{M}^*(s,g,a)
    = G_{s:g,a}^* + \bar{r}_{M}(g,g,a')
    = G_{s:g,a}^* + \bar{r}_{\emptytask}(g,g,a')
    = \bar{Q}_{\varnothing}^*(s,g,a).
\end{align}
\end{proof}

\subsection{A new goal-set based method for composition}

\llemaIsomorphism*

\begin{proof}\label{proof:lemma:homomorphism} \textit{(Lemma \ref{lemma:homomorphism})}
    Bijectivity follows directly from Definition \ref{def:goals_set}.
    We now prove that $\mathcal{F}$ is a homomorphism.
    Let $M_1 \in \mathcal{M}$.
    It suffices to prove the following:
    (i) $\mathcal{F}(M_1)' = \mathcal{F}(\neg M_1)$,
    (ii) $\mathcal{F}(M_1 \land M_2) = \mathcal{F}(M_1) \cap \mathcal{F}(M_2)$ and
    (iii) $\mathcal{F}(M_1 \lor M_2) = \mathcal{F}(M_1) \cup \mathcal{F}(M_2)$.

    \begin{enumerate}
    \item[(i)]
    We first prove the set inclusion $\mathcal{F}(\neg M_1) \subseteq \mathcal{F}(M_1)'$.
    Let $g \in \mathcal{F}(\neg M_1)$, then $\bar{r}_{\neg M_1} (g, g, a) = r_{\neg M_1}(g, a) = \universalreward$.
    Now, $\bar{r}_{\neg (\neg M_1)} (g, g, a) = r_{\neg (\neg M_1)} (g, a) = (r_\universaltask (g, a) + r_\emptytask (g, a)) - r_{\neg M_1} (g, a) = \universalreward + \emptyreward - \universalreward = \emptyreward$. Hence, $g \notin \mathcal{F}(\neg (\neg M_1)) = \mathcal{F}(M_1)$, and thus $g \in \mathcal{F}(M_1)'$.
    The previous reasoning can be traversed in reverse order to obtain $\mathcal{F}(M_1)' \subseteq \mathcal{F}(\neg M_1)$.

    \item[(ii)]
    $\mathcal{F}(M_1) \cap \mathcal{F}(M_2) \subseteq \mathcal{F}(M_1 \land M_2)$ follows directly from the definition of operator $\land$ over $\mathcal{M}$.
    We now prove that $\mathcal{F}(M_1 \land M_2) \subseteq \mathcal{F}(M_1) \cap \mathcal{F}(M_2)$.
    Let $g \in \mathcal{F}(M_1 \land M_2)$, then $\bar{r}_{M_1 \land M_2} (g, g, a) = \universalreward$.
    We now proceed by contradiction, suppose that $g \notin \mathcal{F}(M_1) \cap \mathcal{F}(M_2)$, that is,    
    $\bar{r}_{M_1}(g, g, a) = r_{M_1}(g, a) = \emptyreward$ or $\bar{r}_{M_2}(g, g, a) = r_{M_2}(g, a) = \emptyreward$.
    Then, $\bar{r}_{M_1 \land M_2} (g, g, a) = r_{M_1 \land M_2}(g, a) = \min \{ r_{M_1}(g, a), r_{M_2}(g, a) \} = \emptyreward$, which is a contradiction.
    Consequently $g \in \mathcal{F}(M_1) \cap \mathcal{F}(M_2)$.

    \item[(iii)]
    $\mathcal{F}(M_1) \cup \mathcal{F}(M_2) \subseteq \mathcal{F}(M_1 \lor M_2)$ follows directly from the definition of operator $\lor$ over $\mathcal{M}$.
    We now prove that $\mathcal{F}(M_1 \lor M_2) \subseteq \mathcal{F}(M_1) \cup \mathcal{F}(M_2)$ in an analogous manner to (ii).
    Let $g \in \mathcal{F}(M_1 \lor M_2)$, then $\bar{r}_{M_1 \lor M_2} (g, g, a) = \universalreward$.
    Suppose that $g \notin \mathcal{F}(M_1) \cup \mathcal{F}(M_2)$, that is, $\bar{r}_{M_1} (g, g, a) = \bar{r}_{M_2} (g, g, a) = \emptyreward$.
    Then, $\bar{r}_{M_1 \lor M_2} (g, g, a) = r_{M_1 \lor M_2}(g, a) = \max \{ r_{M_1}(g, a), r_{M_2}(g, a) \} = \emptyreward$ leads to another contradiction, finally giving us that $g \in \mathcal{F}(M_1) \cup \mathcal{F}(M_2)$.
    \end{enumerate}
\end{proof}

Theorem \ref{thm:composition_using_selection} builds on Lemmas \ref{lemma:new_valuefunc} and \ref{lemma:homomorphism} and formalizes the new proposed method for composition.
\thmOne*

\begin{proof} \textit{(Theorem \ref{thm:composition_using_selection})}
    Using Lemma \ref{lemma:new_valuefunc} we write the extended Q-value functions as a selection of state-action slices from the universal and empty tasks.
    We then use the isomorphism defined in Lemma \ref{lemma:homomorphism} to write the set of desired goals for the composite tasks, that is, $\gplus_{\neg M_1}, \gplus_{M_1 \lor M_2}$ and $\gplus_{M_1 \land M_2}$, as composition of the sets of desired goals for tasks $M_1$ and $M_2$, i.e., $\gplus_{\neg M_1} = (\gplus_{M_1}) ' = \mathcal{G} \setminus \gplus_{M_1}$, $\gplus_{M_1 \lor M_2} = \gplus_{M_1} \cup \gplus_{M_2}$ and $\gplus_{M_1 \land M_2} = \gplus_{M_1} \cap \gplus_{M_2}$.
\end{proof}

\subsection{Proof that the optimal slices are ordered.}

\corollaryOne*

\begin{proof} \textit{(Corollary \ref{cor:q-value-empty-less-than-universal})}
    Let $g \in \mathcal{G}$.
    If $g$ is unreachable from $s$, from the proof of Lemma 2 from \cite{tasse2020bta} we know that $\mathcal{M}_{\varnothing, g} = \mathcal{M}_{\mathcal{U}, g}$.
    By uniqueness of the solution to the Bellman optimality equation, their optimal Q-functions must be identical, thus obtaining equality in the statement.
    If $g$ is reachable from $s$, by Corollary 1 from \cite{tasse2020bta}, the extended Q-value function decomposes for all $s, a$ as follows:
    \begin{align*}
        \bar{Q}^*_{\varnothing} (s, g, a) &= G^*_{s:g,a} + \overline{r}_{\emptytask}(g, g, a') \;\; \text{and}\\
        \bar{Q}^*_{\mathcal{U}} (s, g, a) &= G^*_{s:g,a} + \overline{r}_{\universaltask}(g, g, a').
    \end{align*}
    The path-dependent return, $G^*_{s:g,a}$, is identical for both tasks as the optimal path to $g$ is task-independent.
    We now compare the terminal rewards. By definition of the tasks, $\bar{r}_{\emptytask}(g, g, a') = \emptyreward$ and $\bar{r}_{\universaltask}(g, g, a') = \universalreward$. By assumption $\emptyreward \le \universalreward$.
    Since $G^*_{s:g,a}$ is a shared constant, it follows directly that
    $\bar{Q}^*_{\varnothing} (s, g, a) \le \bar{Q}^*_{\mathcal{U}} (s, g, a)$.
\end{proof}

\subsection{Extending the Boolean Task Algebra to stochastic MDPs.}
Here we present the counterexample of extending the BTA framework to stochastic MDPs.
\propCounterex*

\begin{proof}
We will consider the following example MDP:
We have $n$ goals and 1 initial non-goal state $s_0$.
We consider all $2^\mathcal{G}$ possible goal sets. We also consider $|2^\mathcal{G}|$ actions. Each action is associated with a set of states $\tilde{\mathcal{G}}\in 2^\mathcal{G} $, so actions can be indexed accordingly as $a_{\tilde{\mathcal{G}}}$. The MDP is not discounted. Transition probabilities and rewards $r_{s,a}$ for non-terminal states are defined as:
$$r_{s,a_{\tilde{\mathcal{G}}}} = 0.5/|\tilde{\mathcal{G}}|^2, \quad T(s,a_{\tilde{\mathcal{G}}},s') = \begin{cases} 1/|\tilde{\mathcal{G}}| & \textnormal{if } s' \in \tilde{\mathcal{G}} \\ 0 & \textnormal{otherwise.} \end{cases}$$

Let's consider a goal set $\mathcal{G}^+$ and actions $a_{\mathcal{G}^+}$ and $a_{\tilde{\mathcal{G}}},$ with $\mathcal{G}^+ \subset \tilde{\mathcal{G}} $. 
The return for $a_{\mathcal{G}^+}$ will be $-0.5/|\mathcal{G}^+|^2 + 1$, whereas the return for  $a_{\tilde{\mathcal{G}}}$ will be $-0.5/(|\mathcal{G}^+| + \delta )^2 + |\mathcal{G}^+|/ (|\mathcal{G}^+|+\delta),$ with $\delta = |\tilde{\mathcal{G}} \setminus \mathcal{G}^+|$. 
We now look at the difference between these two returns multiplied by $|\mathcal{G}^+|^2 (|\mathcal{G}^+| + \delta)^2$: 
\begin{align}
&(|\mathcal{G}^+|^2 (|\mathcal{G}^+| + \delta)^2) \left(-0.5/|\mathcal{G}^+|^2 + 1 + 0.5/(|\mathcal{G}^+| + \delta )^2 - |\mathcal{G}^+|/ (|\mathcal{G}^+|+\delta)\right) \notag \\
&{} = - 0.5 (|\mathcal{G}^+| + \delta)^2 + |\mathcal{G}^+|^2 (|\mathcal{G}^+|+\delta)^2 +0.5 |\mathcal{G}^+|^2 - |\mathcal{G}^+|^3 (|\mathcal{G}^+|+\delta) \notag \\
&{} = 
|\mathcal{G}^+|^3 \delta + |\mathcal{G}^+|^2 \delta^2 -|\mathcal{G}^+|  \delta - 0.5 \delta ^2 = (|\mathcal{G}^+|^3 - |\mathcal{G}^+|)\delta + (|\mathcal{G}^+|^2 - 0.5) \delta^2 .  \notag 
\end{align}
Since $|\mathcal{G}^+|^3 \ge |\mathcal{G}^+|$ and $|\mathcal{G}^+|^2 > 0.5$ when $|\mathcal{G}^+| \ge 1$ and $\delta>0$, this scaled difference is larger than zero, so choosing an action associated with $\tilde{\mathcal{G}} \supset \mathcal{G}^+ $ is sub-optimal. 
We also know that
\begin{itemize}
\item[$\smallsquare$] Actions $a_{\tilde{\mathcal{G}}}$ with $\tilde{\mathcal{G}} \subset \mathcal{G}^+ $ are suboptimal, as they will achieve the same goal-reaching reward but a more negative immediate reward at $s_0$.
\item[$\smallsquare$] Actions $a_{\tilde{\mathcal{G}}}$ with $\tilde{\mathcal{G}} \not\subset \mathcal{G}^+ $ and $\tilde{\mathcal{G}} \not\supset \mathcal{G}^+ $, always achieve a worse return than actions associated with a set with the same cardinality, but that include more states from $\mathcal{G}^+$ since the expected goal-reaching reward will be higher. So they certainly achieve a worse return than $a_{\mathcal{G}^+}$. 
\end{itemize}
Thus, for any goal set $\mathcal{G}^+$, only the policy that always picks action $a_\mathcal{G}^+$ is optimal.
\end{proof}

\newpage
\section{Additional experimental results}\label{appendix_b}
\subsection{Detailed Experimental Design}
\label{appendix:b:detailed-experimental-design}

Unless stated otherwise, we use the default hyperparameters from the original
environment or algorithm implementations used in the corresponding papers \citep{tasse2020bta, tasse2024skillmachinestemporallogic}.
This appendix records the choices that are specific to our convergence
experiments and the Rooms experiment with stochastic transitions.

\paragraph{Rooms environment.}
\label{app:design-four-rooms}

Training is based on tabular goal-oriented Q-learning.
The trained value functions are the universal task, the empty task, and the $\lceil\log_2 4\rceil=2$ binary basis tasks.
Training uses sparse rewards with shared terminal states.
Goal-oriented Q-learning uses the implementation defaults: discount $1$, learning rate $1$, $\epsilon$-greedy exploration with $\epsilon=1$, and a maximum of $100$ steps per episode.
The training budgets range between $10$ and $5000$ episodes per trained task, as we empirically see that this range is enough for convergence.
Evaluation samples tasks proportionally by task cardinality: with $4$ goals this gives $15$ sampled tasks, of which the empty and universal tasks are excluded from the composed-task evaluation, leaving $13$ tasks.
For each method we evaluate on $1000$ rollouts, with step horizon $100$.

\paragraph{Boxman environment.}
\label{app:design-boxman-sts}
We train four DQN value functions per independent run: \texttt{on}, \texttt{off}, \texttt{blue} ($B$), and \texttt{square} ($S$).
These are evaluated on five composed tasks: $B$, $S$, $B+S$, $B.S$, and $B\operatorname{xor}S$.
We run $3$ independent full runs.
Each DQN receives the RGB observation concatenated with the RGB goal image, giving $6$ input channels.
The network is the original Boxman DQN architecture: convolutional layers $32$ filters of size $8$ and stride $4$, $64$ filters of size $4$ and stride $2$, $64$ filters of size $3$ and stride $1$, followed by a fully connected layer of width $512$ and a linear action head.
We use Adam with learning rate $10^{-4}$, discount $0.99$, batch size $128$, replay buffer size $3\cdot10^5$, learning starts after $10^4$ transitions, training frequency $4$, target network update frequency $1000$, and Huber loss with gradient clipping to $[-1,1]$.
Evaluation uses $100$ episodes per task, method, checkpoint, and completed run, with maximum trajectory length $20$.

\paragraph{Office-World environment.}
\label{app:design-office}

Office-World uses tabular value functions.
The evaluation tasks are Coffee, Patrol, Coffee and Mail, and Long and the primitive library is
$
    \mathcal{B}
    =
    \{0,1\}
    \cup
    \{p_\varphi:\varphi\in\{A,B,C,D,c,d,m,o,tm,to\}\}
    \cup
    \{c_d\},
$
so each run trains $13$ tabular world value functions.
The propositions $A,B,C,D$ denote rooms, $c$ coffee, $m$ mail, $o$ office, $tm$ and $to$ the dynamic mail/person availability propositions, and $d$ the decoration constraint.
We use Q-learning with learning rate $0.1$, discount $0.9$, $\epsilon$-greedy exploration with $\epsilon=0.5$, and zero initialization.
We run $3$ independent seeds.
Evaluation uses $50$ episodes per task, method, and checkpoint, with horizon $1000$ and evaluation discount $0.9$.
We compare Original, Univ./empty, and Base tasks; the first two use only the learned $0$ and $1$ primitives, while Base tasks can use all primitives in $\mathcal{B}$.
The LTL specifications in the environment are:
\begin{enumerate}[leftmargin=2.5cm]
    \item[\textbf{Coffee}:] Deliver coffee to the office without breaking any decorations.
    \item[\textbf{Patrol}:] Visit rooms A, B, C, and D in sequence without breaking any decorations.
    \item[\textbf{Coffee \& Mail}:] Deliver both coffee and mail to the office, in either order, without breaking any decorations.
    \item[\textbf{Long}:] Deliver mail until none remains, then deliver coffee while people are present in the office, and finally patrol rooms A, B, C, D, and A, without breaking any decorations.
\end{enumerate}

\paragraph{Safety-Gym environment.}
\label{app:design-safety-gym}

Safety-Gym is the continuous-control function-approximation experiment.
The learned primitive library is
$
    \{0,1,p_{\textsc{buttons}},p_{\textsc{goal}},
    p_{\textsc{hazards}},c_{\textsc{hazards}}\},
$
so each run trains $6$ world value functions.
Each primitive is trained with TD3.
The actor and critic use hidden layers $[2024,2024,2024]$ and Gaussian action noise has standard deviation $0.2$.
The main training parameters are learning rate $10^{-6}$, discount $0.99$, batch size $32$ and a replay buffer size $10^6$ for $2 \times 10^6$ steps.
We run $2$ independent seeds and evaluation uses $500$ episodes per task, method, and checkpoint, with horizon $1000$ and evaluation discount $0.99$.
We now describe the LTL tasks from this experiment as defined in \cite{tasse2024skillmachinestemporallogic}:
\begin{enumerate}
    \item[\textbf{Task 1}]: Navigate to a button and then to a cylinder.
    \item[\textbf{Task 2}]: Navigate to a button and then to a cylinder while never entering blue regions.
    \item[\textbf{Task 3}]: Navigate to a button, then to a cylinder without entering blue regions, then to a button inside a blue region, and finally to a cylinder again.
    \item[\textbf{Task 4}]: Navigate to a button and then to a cylinder in a blue region.
    \item[\textbf{Task 5}]: Navigate to a cylinder, then to a button in a blue region, and finally to a cylinder again.
    \item[\textbf{Task 6}]: Navigate to a blue region, then to a button with a cylinder, and finally to a cylinder while avoiding blue regions.
\end{enumerate}

\subsection{Additional experimental results}
\subsubsection{Additional results for Rooms environment}
\paragraph{Illustrating learned policies.}
In Figure \ref{fig:learned_policies} we illustrate the policies derived from the learned value functions for the rooms environment.
Figure \ref{fig:1} shows the optimal policy derived from maximizing over actions in the slice corresponding to a given goal $g$, that is, $\bar{Q}_\mathcal{U}^* (\cdot, g, \cdot)$.
Similarly, Figure \ref{fig:2} illustrates the optimal policy derived from not following $g$, that is, $\bar{Q}_\varnothing^* (\cdot, g, \cdot)$.
These goal slices are then combined via maximization over the goal and actions dimensions to obtain the policies illustrated in Figures \ref{fig:3} and \ref{fig:4}.
We note that both the slices and the global policies are similar in shape, the key difference being the range of rewards associated to each of them.

\begin{figure}[htb]
    \centering
    \begin{minipage}{0.23\textwidth} 
        \centering
        \includegraphics[width=\linewidth]{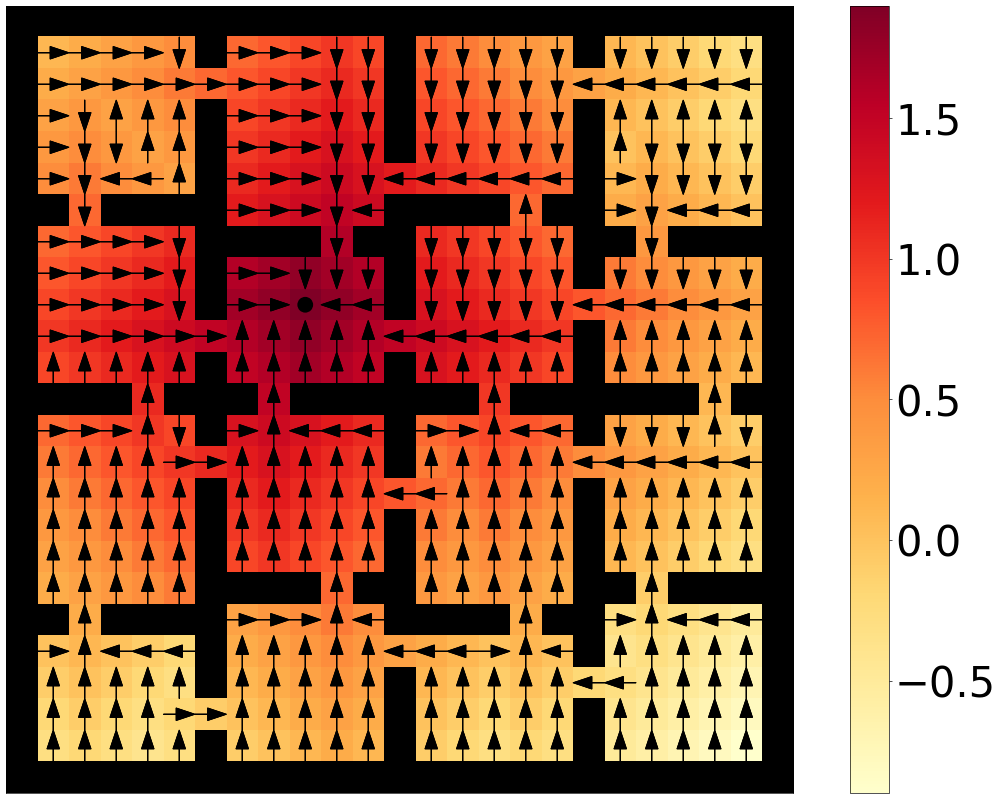}
        \subcaption{Optimal policy for pursuing $g$.}
        \label{fig:1}
    \end{minipage}
    \hfill
    \begin{minipage}{0.24\textwidth}
        \centering
        \includegraphics[width=\linewidth]{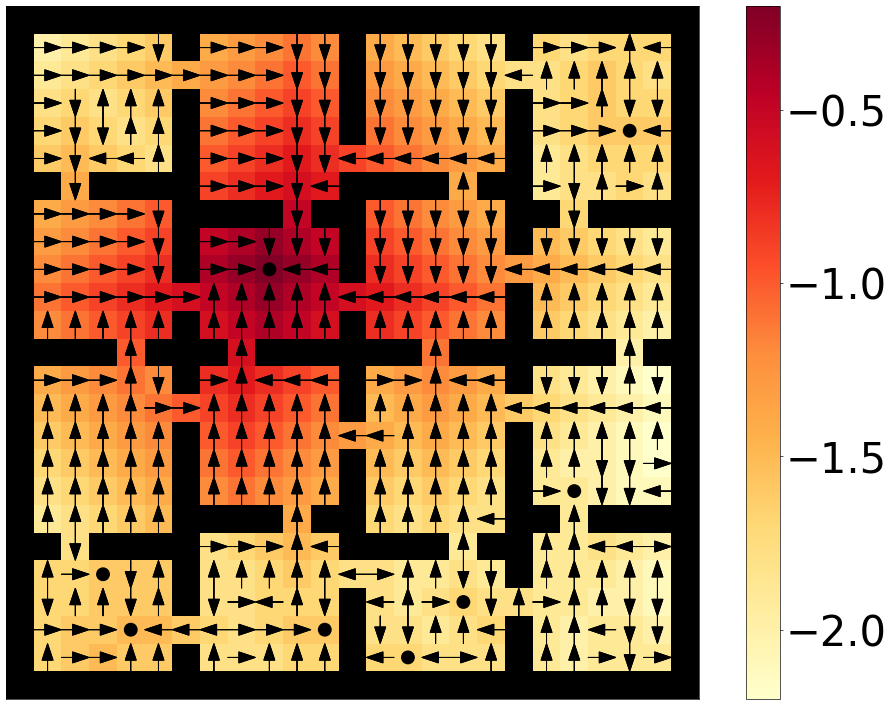}
        \subcaption{Optimal policy for not pursuing $g$.}
        \label{fig:2}
    \end{minipage}
    \hfill
    \begin{minipage}{0.23\textwidth}
        \centering
        \includegraphics[width=\linewidth]{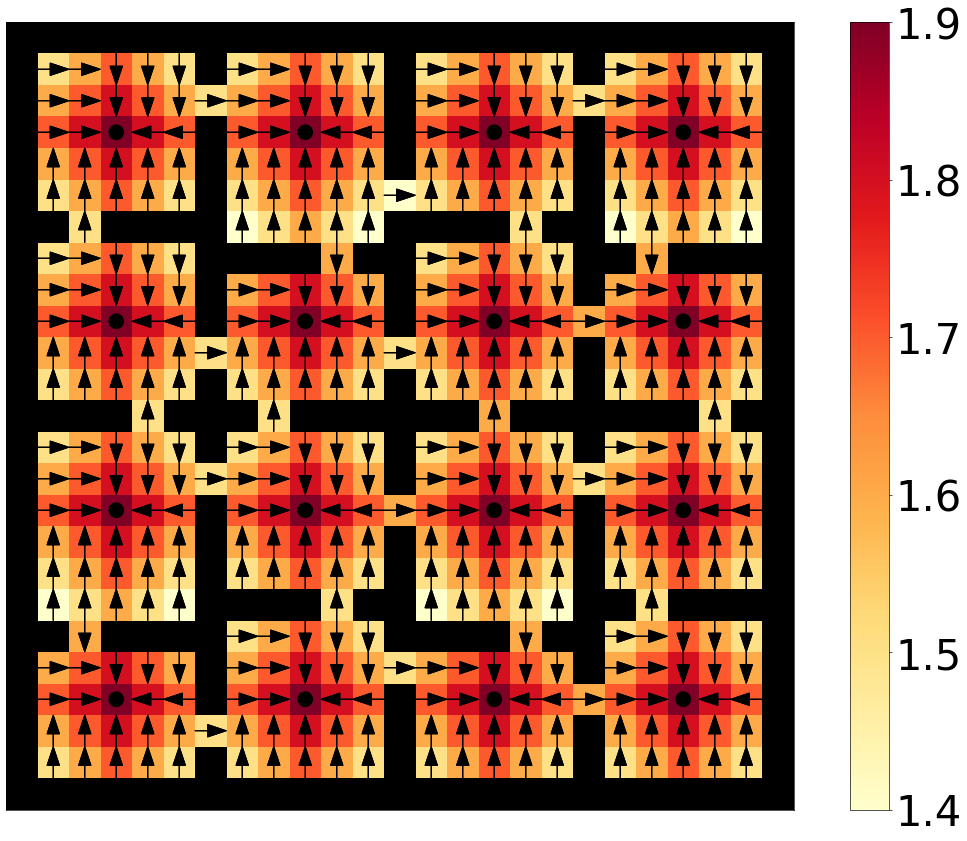}
        \subcaption{Optimal policy for pursuing any goal in $\mathcal{G}$.}
        \label{fig:3}
    \end{minipage}
    \hfill
    \begin{minipage}{0.24\textwidth}
        \centering
        \includegraphics[width=\linewidth]{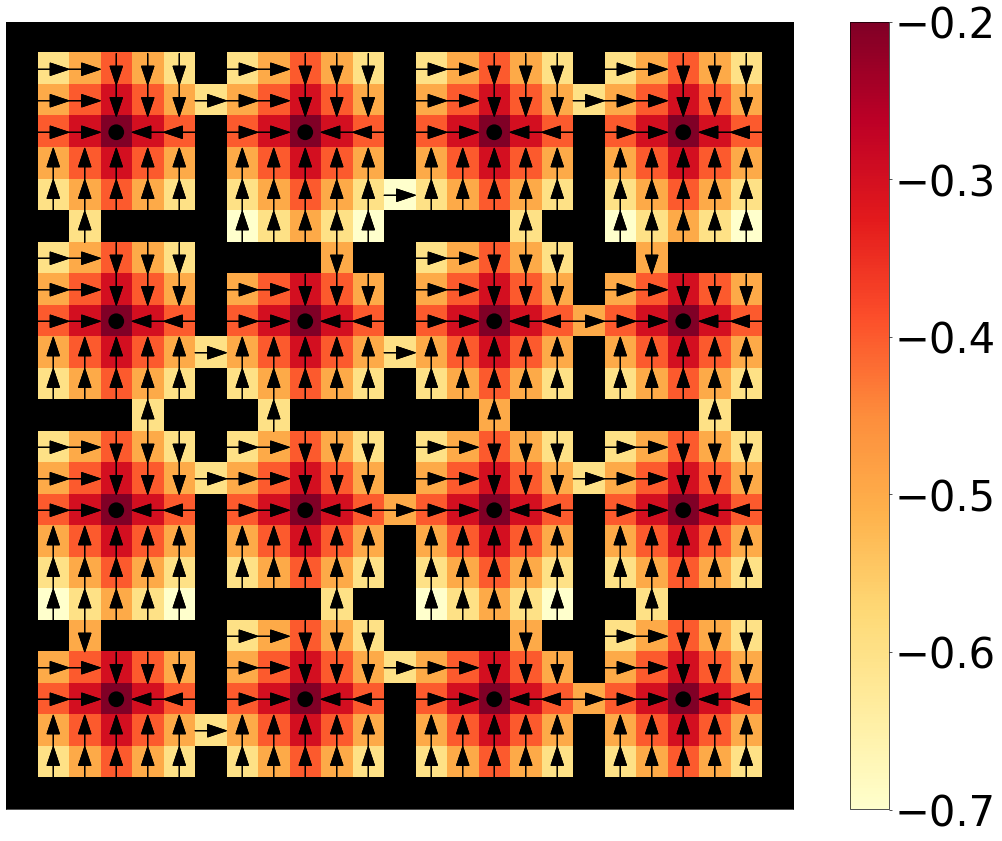}
        \subcaption{Optimal policy for not pursuing any goal in $\mathcal{G}$.}
        \label{fig:4}
    \end{minipage}
    \caption{Optimal policies derived from the learned value functions for the 16 goals environment.}
    \label{fig:learned_policies}
\end{figure}

\paragraph{Performance upon convergence.}
Figure \ref{fig:returns_compared} provides a side-by-side comparison of the returns obtained with the original and proposed composition methods.
We observe that the returns obtained for both methods are identical, given that we have provided enough maximum iterations for convergence and confirm that theoretical optimality also holds in a practical setting.
We also observe that expected returns are higher as the number of goals in the task increases, which can be explained by the fact that the agent requires fewer steps on average to reach any of the goals, hence incurring lower time-step penalties.

\begin{figure}[h]
    \centering
    \includegraphics[width=0.95\linewidth]{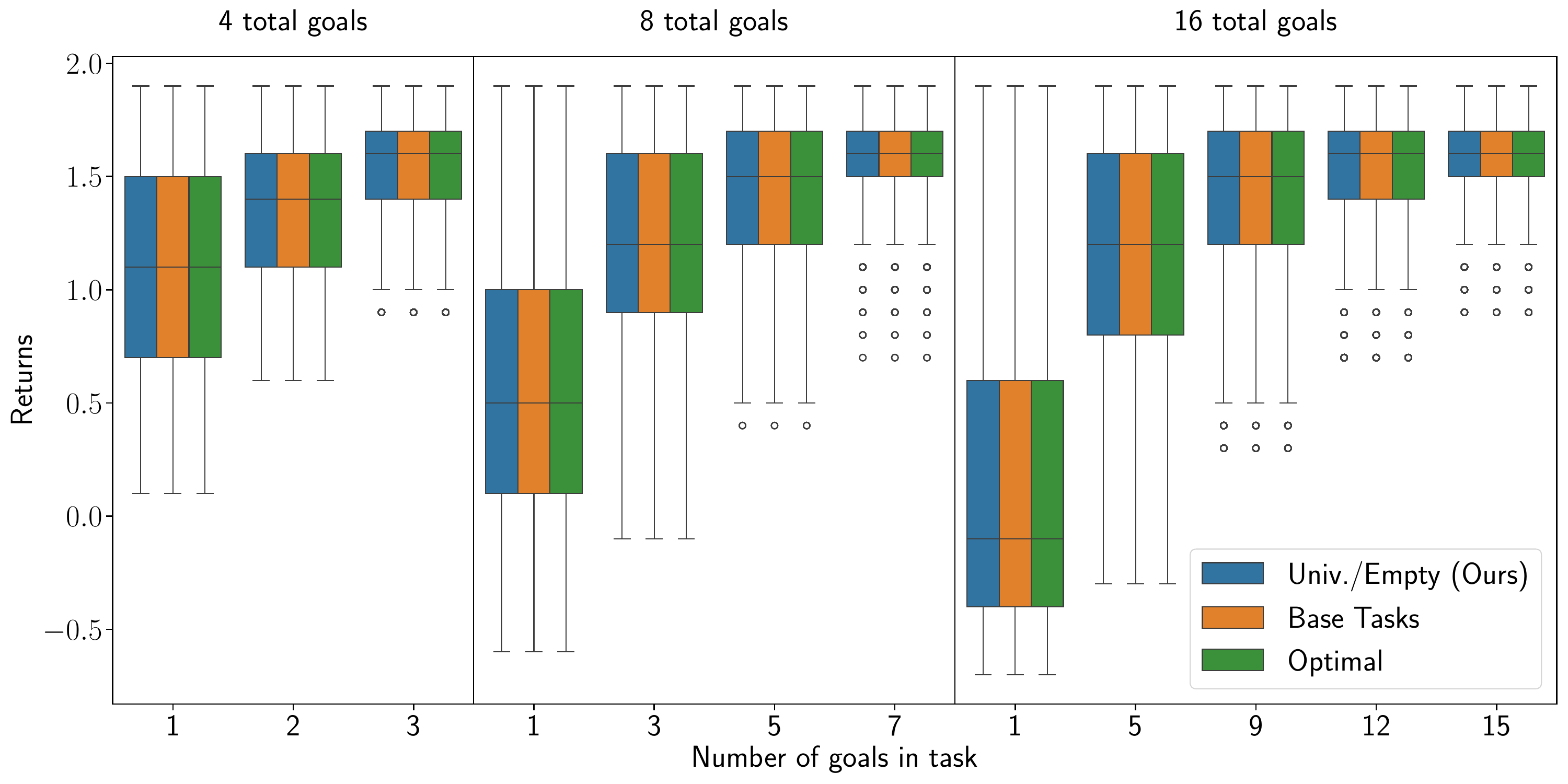}
    \caption{
    Boxplots showing the quartiles of the returns of the original composition method (using base tasks) and composition using the universal and empty tasks.
    We aggregate the results per goal set length.
    The returns for each task are measured on $10^4$ episodes and the value functions are obtained with $5000$ maximum iterations of the training algorithm.
    }
    \label{fig:returns_compared}
\end{figure}

\begingroup
\subsubsection{Rooms environment with stochastic transitions}
\label{app:stochastic-rooms}

\paragraph{Experimental setup.}
We evaluate slip probabilities $p \in \{0, 0.05, 0.1, 0.2, 0.3, 0.5\}$ in the Rooms environments with $|\mathcal{G}| \in \{4,8,16\}$.
At each movement step, with probability $p$ the intended direction is replaced by a direction sampled uniformly from the other three directions.
Slip only affects movement actions. The \textsc{stay} action, which collects a goal, is never perturbed.

For each slip probability, we train all value functions from scratch using tabular goal-oriented Q-learning.
These correspond to the universal task, with terminal reward $r_\mathcal{U}=2$ at every goal, the empty task, with terminal reward $r_\varnothing=-0.1$ at every goal, and $\lceil\log_2|\mathcal{G}|\rceil$ binary basis tasks, giving $2$, $3$ and $4$ basis tasks, respectively.
We use learning rate $\alpha=0.1$, discount $\gamma=1$, a uniform random behaviour policy, and a step reward of $-0.1$.
The penalty for terminating at a goal other than the conditioning goal is
$
    \bar{r}_{\min}
    = \min\bigl(r_\varnothing,
        (r_\varnothing-r_\mathcal{U})\operatorname{diam}\bigr),
$
where $\operatorname{diam}$ is the diameter of the map.
Each task is trained for $2\times10^5$ steps when $|\mathcal{G}|=4$ and $4\times10^5$ steps when $|\mathcal{G}|\in\{8,16\}$ and we use the full training budget without early stopping.
The goal-set method uses the universal and empty value functions, while the original BTA additionally uses the learned basis tasks, giving $4$, $5$ and $6$ learned value functions in total.

For $|\mathcal{G}|\leq6$, we evaluate all $2^{|\mathcal{G}|}-2$ non-trivial goal subsets, giving $14$ tasks for the four-goal map.
For larger goal sets, we sample subsets by cardinality, taking $\min(50,\binom{|\mathcal{G}|}{k})$ subsets of each size $k\in\{1,\ldots,|\mathcal{G}|-1\}$.
This gives $222$ tasks for $|\mathcal{G}|=8$ and $682$ tasks for $|\mathcal{G}|=16$.
For each task and method, we evaluate $200$ Monte Carlo rollouts from initial states sampled uniformly and with a maximum horizon of $H=100$ steps.
Results are averaged over initial states, then tasks, and finally $3$ independent seeds.

\begin{figure}[htbp]
                \centering
    \begin{subfigure}[b]{0.48\textwidth}
        \centering
        \includegraphics[width=\linewidth]{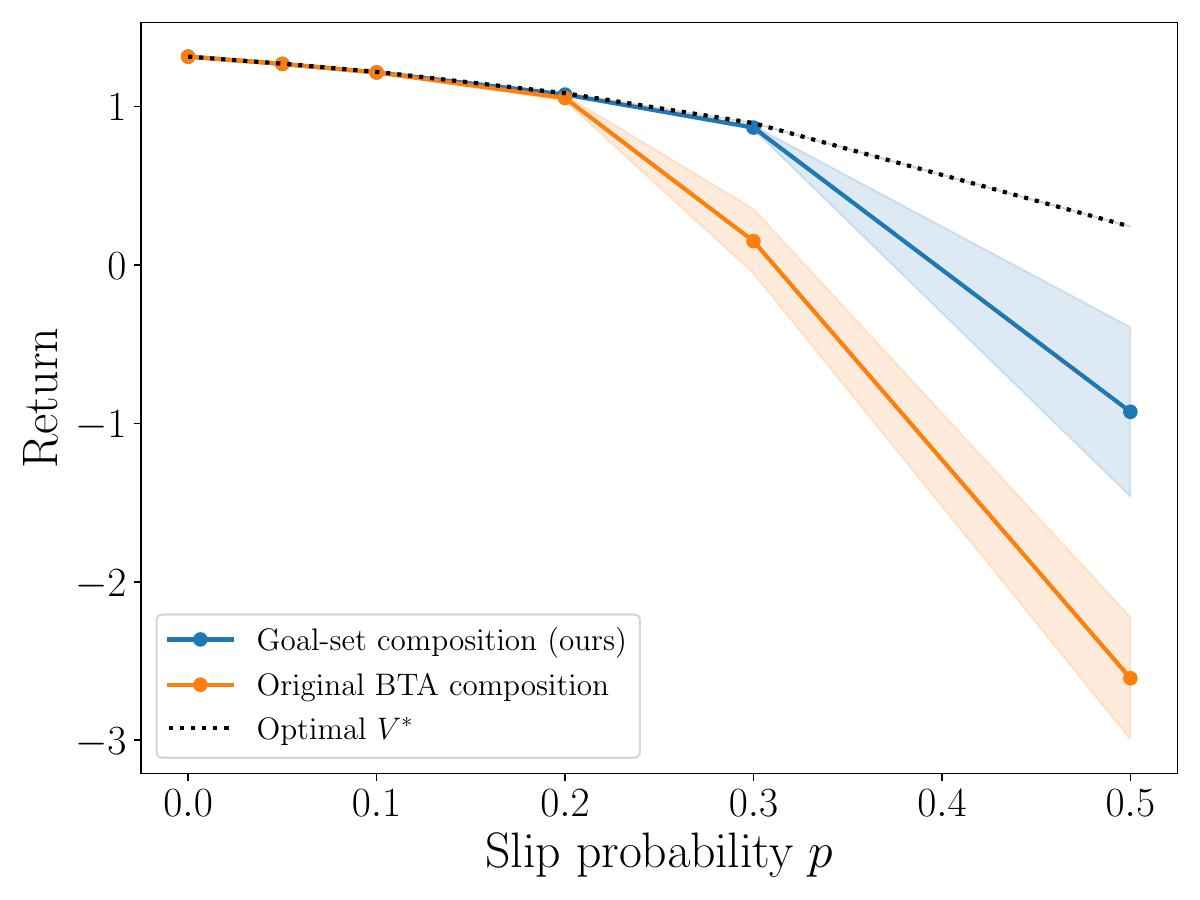}
        \caption{$|\mathcal{G}|=4$.}
    \end{subfigure}
    \hfill
    \begin{subfigure}[b]{0.48\textwidth}
        \centering
        \includegraphics[width=\linewidth]{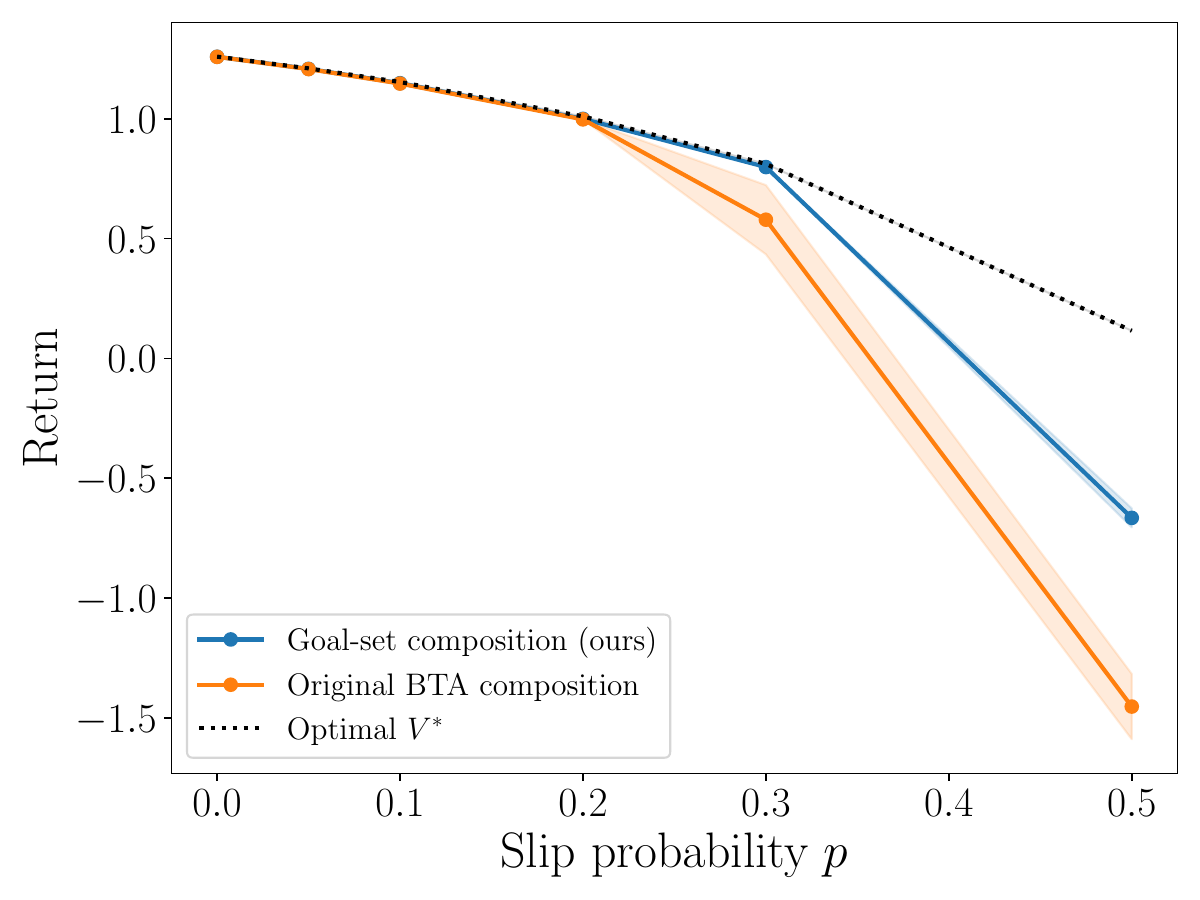}
        \caption{$|\mathcal{G}|=8$.}
    \end{subfigure}
    \caption{Average return as a function of slip probability in the Rooms environments with $|\mathcal{G}|\in\{4,8\}$.
    Solid lines show the returns of policies composed using the goal-set method and base tasks (original BTA).
    The dotted black line shows the optimal value function $V^*$ averaged over the evaluated tasks and initial states.
    }
    \label{fig:slip_returns_small}
\end{figure}

\paragraph{Results.}
Figures~\ref{fig:slip_returns} and~\ref{fig:slip_returns_small} show similar trends across the different environment sizes.
The two composition methods obtain similar returns at low slip probabilities.
As slip increases, the original BTA curve decreases more sharply than the goal-set curve, leading to a growing difference between the methods.
The dotted optimal-value curve also decreases with slip, showing that stochastic transitions reduce the achievable return even under an optimal policy.
Both composition methods remain close to this reference at low slip probabilities.
At the highest slip probability, both methods fall below the optimal return, but goal-set composition has a smaller gap than the original BTA across all three environment sizes.

To justify this phenomenon we analyze the effect of NOT operations of the task definition in the acheived return.
Table~\ref{tab:negation-excess-gap} reports the mean excess optimality gap of BTA over goal-set (GS) composition, paired on the same task $M$, slip probability $p$, and seed $s$.
We present results for Rooms with $|\mathcal{G}|\in\{4,8,16\}$; the larger settings allow longer task expressions with more negations and provide a clearer picture of their effect.
The optimal return cancels,
\begin{equation}
\label{eq:excess-gap}
\Delta_{M,p,s}
= \bigl(V^*_{M,p}-V^{\pi^{\mathrm{BTA}}_s}_{M,p}\bigr)
- \bigl(V^*_{M,p}-V^{\pi^{\mathrm{GS}}_s}_{M,p}\bigr)
= V^{\pi^{\mathrm{GS}}_s}_{M,p}-V^{\pi^{\mathrm{BTA}}_s}_{M,p}.
\end{equation}
The four-goal setting has only 1--6 tasks per bin (see $n$), too few to establish a representative trend across negation levels. We therefore focus on the 16-goal setting, which has the most observations overall, although its extreme bins remain sparse.

We see that positive values favor goal-set composition.
The excess gap is small at low slip probabilities and, for $|\mathcal{G}|\in\{8,16\}$ at $p=0.5$, increases with the proportion of negated literals, consistent with greater sensitivity of negation-based composition to estimation error.

\begin{table}[htbp]
    \centering
        \caption{Mean excess optimality gap in Rooms with $|\mathcal{G}|\in\{4,8,16\}$, grouped by the percentage of negated literals in the BTA expression. Columns give slip probability $p$ and $n$ is the number of observations in each bin. Blue favors goal-set composition; orange favors BTA. Color intensity (orange for low, blue for high) is normalized independently within each probability column of each Rooms setting.}
    \label{tab:negation-excess-gap}
    \small
    \definecolor{gapPositive}{HTML}{3182BD}
    \definecolor{gapNegative}{HTML}{E67E22}
    \newcommand{\gapcell}[3]{\cellcolor{#1!#2}\ensuremath{#3}}
    \renewcommand{\arraystretch}{1.18}
    \begin{tabular*}{\linewidth}{@{\extracolsep{\fill}}lrrrrrrr}
        \toprule
        \textbf{NOT (\%)} & $n$ & $p=0.00$ & $p=0.05$ & $p=0.10$ & $p=0.20$ & $p=0.30$ & $p=0.50$ \\
        \midrule
        \multicolumn{8}{c}{\textit{Rooms: $|\mathcal{G}|=4$ \quad (2 base tasks)}} \\
        \midrule
        0--20   & 1    & \gapcell{gapPositive}{0.0}{0.000} & \gapcell{gapPositive}{0.0}{0.000} & \gapcell{gapNegative}{32.5}{-0.001} & \gapcell{gapNegative}{0.2}{-0.001} & \gapcell{gapNegative}{5.5}{-0.182} & \gapcell{gapPositive}{20.6}{0.656} \\
        20--40  & 3    & \gapcell{gapPositive}{0.0}{0.000} & \gapcell{gapPositive}{0.0}{0.000} & \gapcell{gapPositive}{0.0}{-0.000} & \gapcell{gapPositive}{0.0}{0.000} & \gapcell{gapPositive}{4.8}{0.161} & \gapcell{gapPositive}{65.0}{2.068} \\
        40--60  & 6    & \gapcell{gapPositive}{0.0}{0.000} & \gapcell{gapPositive}{0.0}{-0.000} & \gapcell{gapPositive}{0.0}{0.000} & \gapcell{gapPositive}{0.0}{0.000} & \gapcell{gapPositive}{19.7}{0.656} & \gapcell{gapPositive}{56.8}{1.807} \\
        60--80  & 3    & \gapcell{gapPositive}{0.0}{0.000} & \gapcell{gapPositive}{0.0}{-0.000} & \gapcell{gapPositive}{32.5}{0.001} & \gapcell{gapPositive}{0.0}{-0.000} & \gapcell{gapPositive}{36.4}{1.212} & \gapcell{gapPositive}{53.6}{1.706} \\
        80--100 & 1    & \gapcell{gapPositive}{0.0}{0.000} & \gapcell{gapNegative}{65.0}{-0.001} & \gapcell{gapPositive}{65.0}{0.002} & \gapcell{gapPositive}{65.0}{0.307} & \gapcell{gapPositive}{65.0}{2.163} & \gapcell{gapPositive}{22.8}{0.725} \\
        \midrule
        \multicolumn{8}{c}{\textit{Rooms: $|\mathcal{G}|=8$ \quad (3 base tasks)}} \\
        \midrule
        0--20   & 4    & \gapcell{gapPositive}{9.3}{0.001} & \gapcell{gapPositive}{54.2}{0.005} & \gapcell{gapPositive}{0.0}{-0.000} & \gapcell{gapPositive}{65.0}{0.008} & \gapcell{gapPositive}{3.5}{0.019} & \gapcell{gapNegative}{14.9}{-0.330} \\
        20--40  & 38   & \gapcell{gapPositive}{0.0}{0.000} & \gapcell{gapPositive}{10.8}{0.001} & \gapcell{gapPositive}{8.1}{0.002} & \gapcell{gapPositive}{24.4}{0.003} & \gapcell{gapPositive}{16.8}{0.091} & \gapcell{gapPositive}{12.9}{0.286} \\
        40--60  & 159  & \gapcell{gapPositive}{0.0}{0.000} & \gapcell{gapPositive}{0.0}{0.000} & \gapcell{gapPositive}{4.1}{0.001} & \gapcell{gapPositive}{16.2}{0.002} & \gapcell{gapPositive}{39.2}{0.213} & \gapcell{gapPositive}{36.6}{0.810} \\
        60--80  & 48   & \gapcell{gapPositive}{9.3}{0.001} & \gapcell{gapPositive}{10.8}{0.001} & \gapcell{gapPositive}{8.1}{0.002} & \gapcell{gapPositive}{24.4}{0.003} & \gapcell{gapPositive}{65.0}{0.353} & \gapcell{gapPositive}{54.4}{1.205} \\
        80--100 & 4    & \gapcell{gapPositive}{65.0}{0.007} & \gapcell{gapPositive}{65.0}{0.006} & \gapcell{gapPositive}{65.0}{0.016} & \gapcell{gapPositive}{48.8}{0.006} & \gapcell{gapPositive}{58.2}{0.316} & \gapcell{gapPositive}{65.0}{1.440} \\
        \midrule
        \multicolumn{8}{c}{\textit{Rooms: $|\mathcal{G}|=16$ \quad (4 base tasks)}} \\
        \midrule
        0--20   & 5    & \gapcell{gapPositive}{65.0}{0.044} & \gapcell{gapPositive}{12.5}{0.010} & \gapcell{gapPositive}{26.7}{0.016} & \gapcell{gapPositive}{4.3}{0.012} & \gapcell{gapPositive}{6.9}{0.032} & \gapcell{gapPositive}{5.5}{0.200} \\
        20--40  & 146  & \gapcell{gapPositive}{39.9}{0.027} & \gapcell{gapPositive}{38.8}{0.031} & \gapcell{gapPositive}{40.0}{0.024} & \gapcell{gapPositive}{12.4}{0.035} & \gapcell{gapPositive}{21.8}{0.101} & \gapcell{gapPositive}{14.1}{0.509} \\
        40--60  & 1497 & \gapcell{gapPositive}{11.8}{0.008} & \gapcell{gapPositive}{17.5}{0.014} & \gapcell{gapPositive}{20.0}{0.012} & \gapcell{gapPositive}{7.8}{0.022} & \gapcell{gapPositive}{22.9}{0.106} & \gapcell{gapPositive}{23.5}{0.846} \\
        60--80  & 184  & \gapcell{gapPositive}{44.3}{0.030} & \gapcell{gapPositive}{43.8}{0.035} & \gapcell{gapPositive}{63.3}{0.038} & \gapcell{gapPositive}{25.9}{0.073} & \gapcell{gapPositive}{48.8}{0.226} & \gapcell{gapPositive}{37.3}{1.343} \\
        80--100 & 7    & \gapcell{gapPositive}{51.7}{0.035} & \gapcell{gapPositive}{65.0}{0.052} & \gapcell{gapPositive}{65.0}{0.039} & \gapcell{gapPositive}{65.0}{0.183} & \gapcell{gapPositive}{65.0}{0.301} & \gapcell{gapPositive}{65.0}{2.343} \\
        \bottomrule
    \end{tabular*}
\end{table}

\par
\endgroup

\subsubsection{Additional results for Boxman environment}
\label{app:boxman-results}
Figure \ref{fig:pertask_boxman} shows per task returns and successes for the Boxman environment.
We confirm a similar pattern among all tasks, where the base tasks method converges faster to the local optimum but both methods eventually stabilize at a similar performance level.
\begin{figure}[h]
    \centering
    \includegraphics[width=1\linewidth]{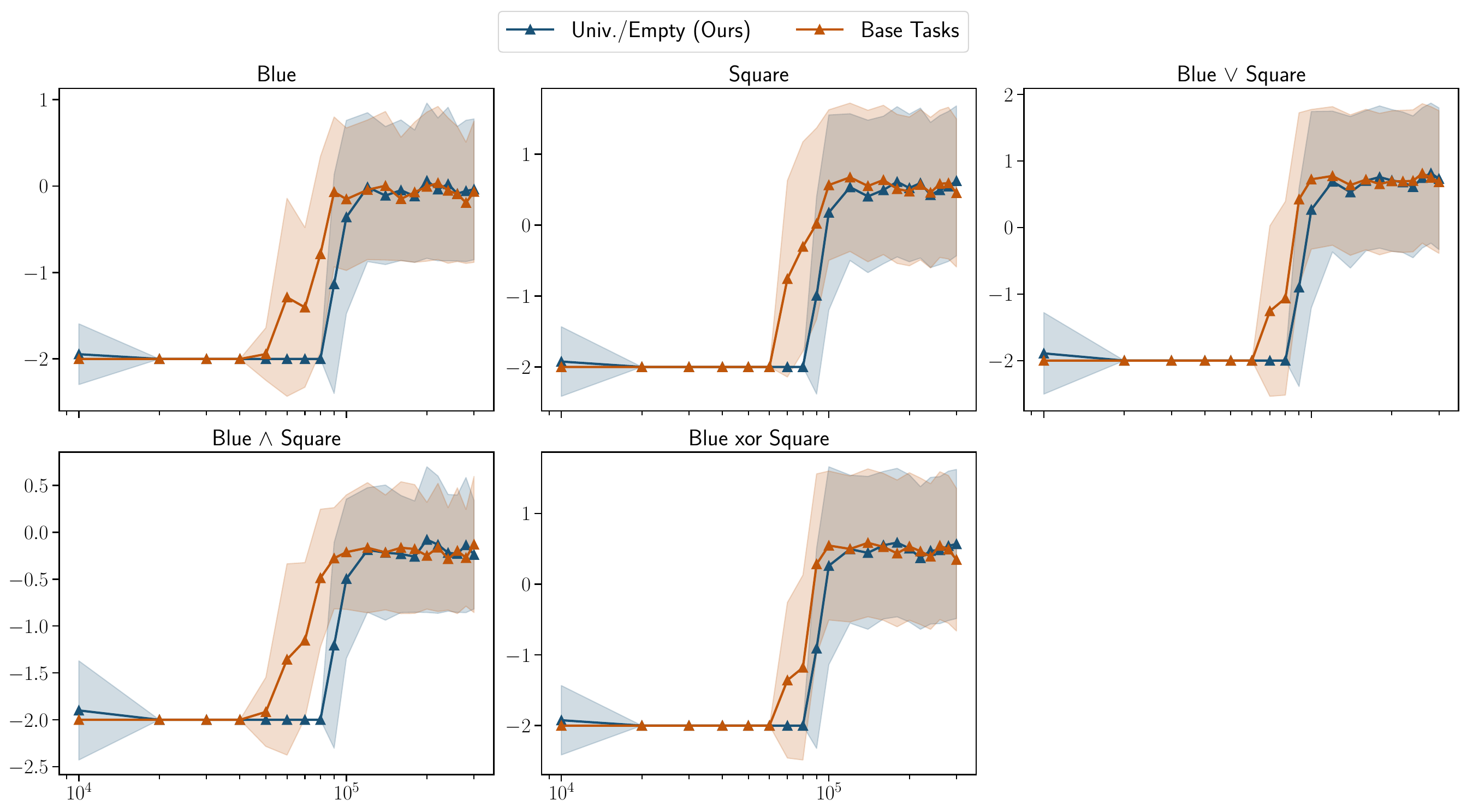}
    \caption{Average returns ($\pm$ std) for the boxman environment. The x-axis shows the number of training iterations per UVFA, hence base tasks requires twice as much compute power given that it needs to train more value functions.}
    \label{fig:pertask_boxman}
\end{figure}

\begingroup
\paragraph{Stochastic transitions.}
Figure~\ref{fig:boxman_slip_returns} compares goal-set composition using the universal and empty tasks with base-task composition in Boxman for slip probabilities from $0$ to $0.5$.
In this environment with function approximation, the two methods obtain nearly identical returns throughout the sweep, and both remain below the optimal-value reference.
Thus, the performance advantage of goal-set composition observed in the tabular Rooms experiments disappears in this setting.
The negation dependence in Table~\ref{tab:negation-excess-gap} offers a possible explanation, only one of the five Boxman tasks uses negation, and it contains just one negation, limiting exposure to the effect observed in Rooms.

\begin{figure}[htbp]
    \centering
    \includegraphics[width=0.65\linewidth]{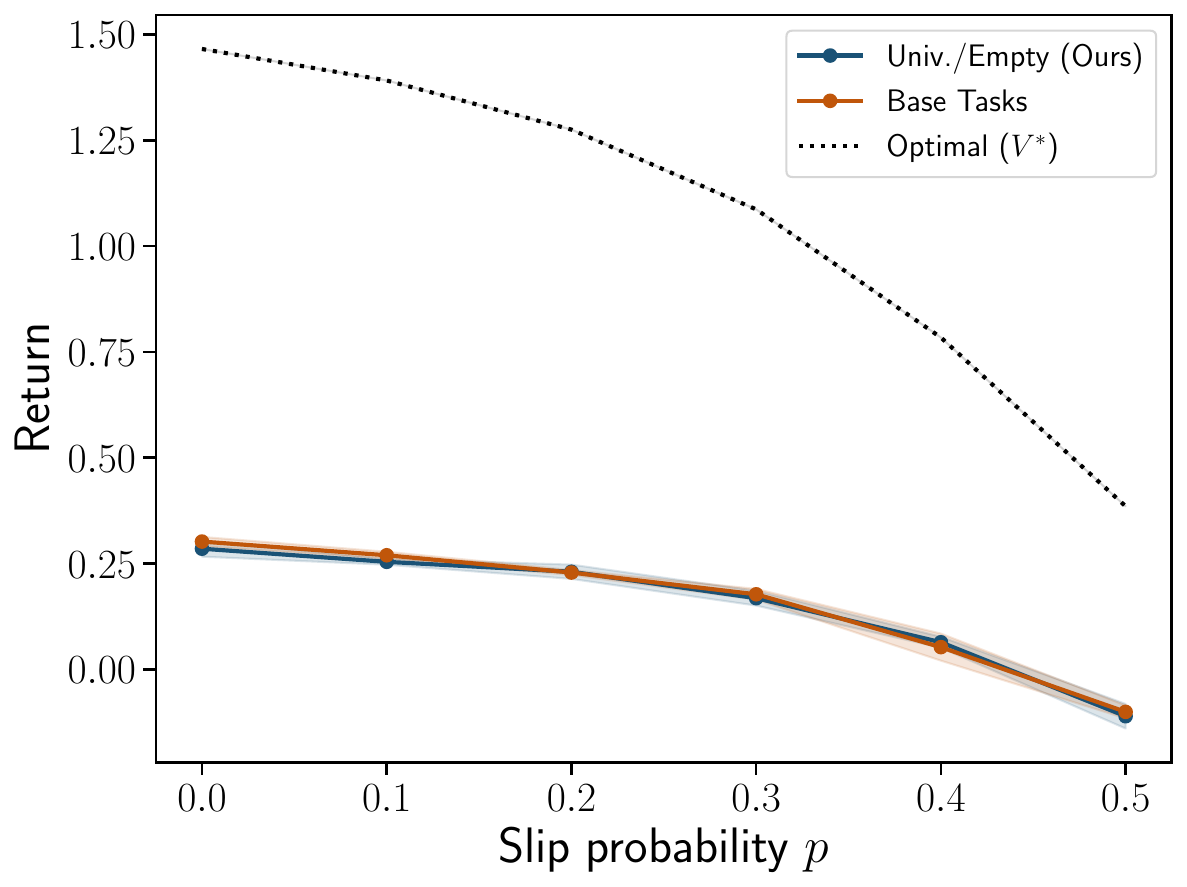}
    \caption{Average return as a function of slip probability in Boxman with function approximation.
    Solid lines show goal-set composition using the universal and empty tasks (\textsc{Univ./Empty}) and base-task composition (\textsc{Base Tasks}).
    The dotted black line shows the optimal-value reference $V^*$.
    Both composition methods exhibit similar returns across slip probabilities.}
    \label{fig:boxman_slip_returns}
\end{figure}
\par
\endgroup

\subsubsection{Additional results for Office Gridworld environment}\label{appendix:b:office}
Figure \ref{fig:skill_machines_gridworld_pertask} shows per task returns and successes for the Office Gridworld domain.
It is more informative to look at the success graphs in Figure \ref{fig:skill_machines_gridworld_pertask:b}, where we see that the percentage of successes in the LTL specification converged to 1, meaning that the agent correctly learns to execute the task.
It is important to note that we re-use the universal and empty value functions learned during the base task learning, hence there is a correlation in the performance of all three methods given by the particularities of the three runs (e.g. sudden collapse Figure \ref{fig:skill_machines_gridworld_pertask:a} for the Long and Patrol tasks, or the jittery pattern of Coffee and Mail).
However, if we attend to the average presented in Figure \ref{fig:office:returns_successes} we see a more stable convergence across iterations.
\begin{figure}[h]
    \centering
    \begin{subfigure}[b]{\textwidth}
        \centering
        \includegraphics[width=\textwidth]{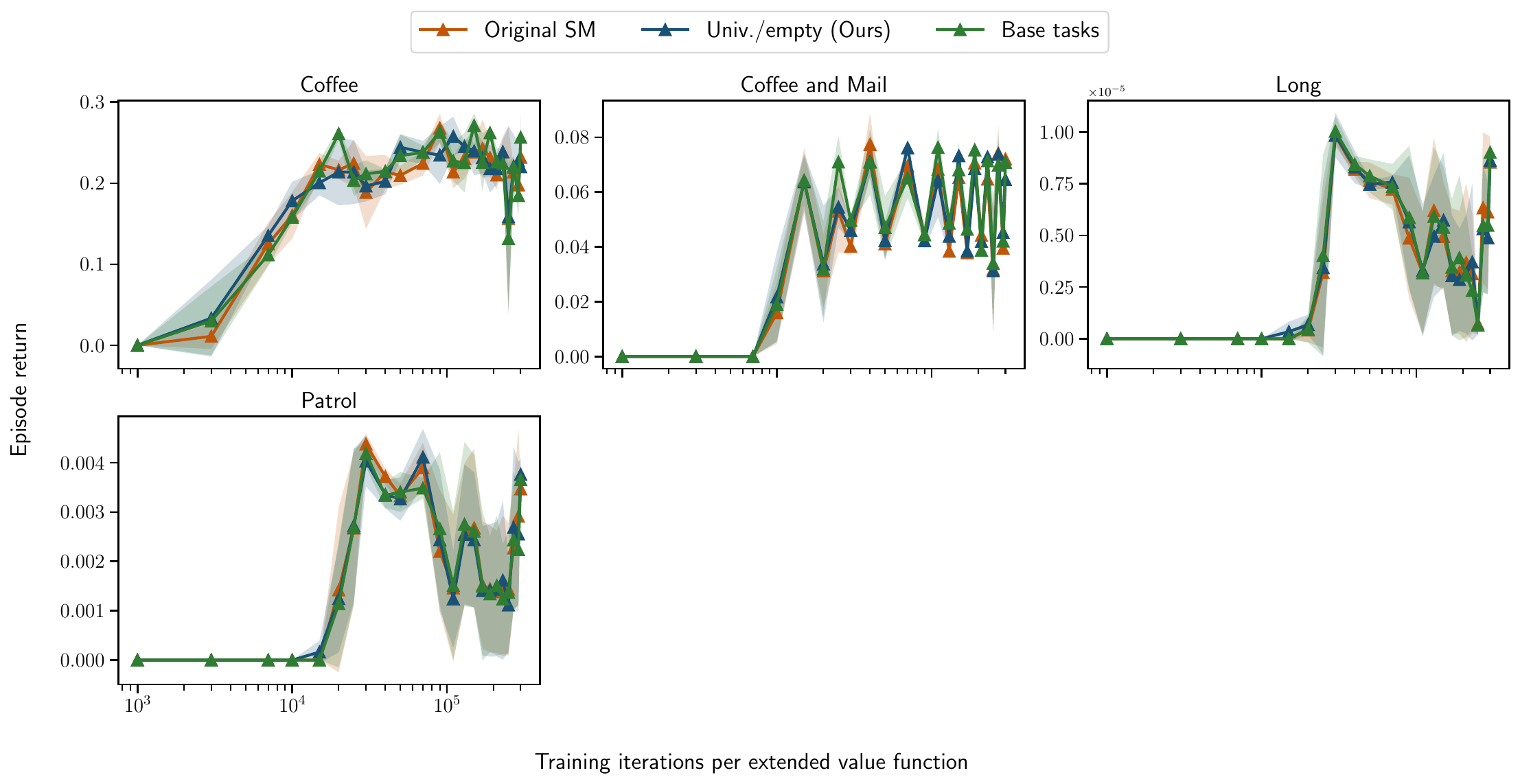}
        \caption{Total rewards disaggregated per task.}
        \label{fig:skill_machines_gridworld_pertask:a}
    \end{subfigure}
    \newline 
    \begin{subfigure}[b]{\textwidth}
        \centering
        \includegraphics[width=\textwidth]{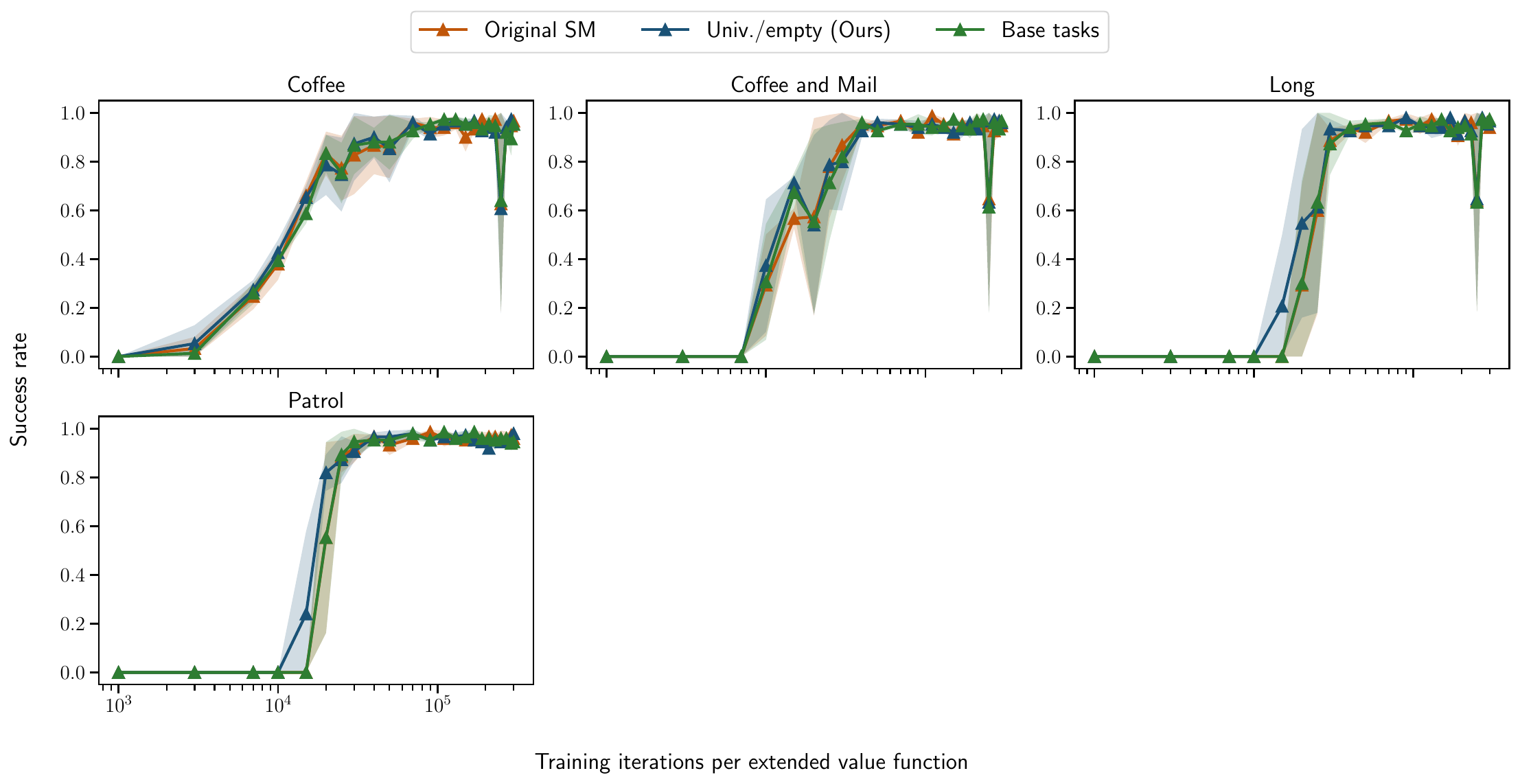}
        \caption{Total successes of the LTL specification  disaggregated per task.}
        \label{fig:skill_machines_gridworld_pertask:b}
    \end{subfigure}
    \caption{Evolution of average ($\pm$ std) successes and returns across runs in the temporal composition of actions for the main 4 tasks of the Office Gridworld.
    Results are reported over $3$ runs.
    See Appendix \ref{appendix:b:detailed-experimental-design} for explanation on the LTL task specifications.
    }
    \label{fig:skill_machines_gridworld_pertask}
\end{figure}

\subsubsection{Additional results for Safety gym environment}\label{appendix:b:safety}
We ran the Safety Gym experiment in order to evaluate the temporal-composition setting with continuous control and function approximation. Figure~\ref{fig:safety_gym_results} shows the evolution of the average return and success rate as a function of the number of training iterations per extended value function. The three methods exhibit similar learning trends, with performance improving as training progresses. This confirms that the additional training effort required to learn extra base tasks does not translate into a better overall performance as compared to using only the universal and empty tasks.
Moreover, this also means that the improved composition cost does not translate to a worse performance (Original vs Univ./Empty).

Figure~\ref{fig:safety_composition_time} complements these results by reporting the composition-time distribution for the same Safety Gym setting. The Univ./empty construction substantially reduces the cost of composing extended value functions, while preserving comparable return and success performance. Thus, the Safety Gym experiment supports the main empirical conclusion: using only the universal and empty tasks improves computational efficiency without degrading the quality of the composed policy.

\begin{figure}[t]
    \centering
    \begin{subfigure}{0.48\linewidth}
        \centering
        \includegraphics[width=\linewidth]{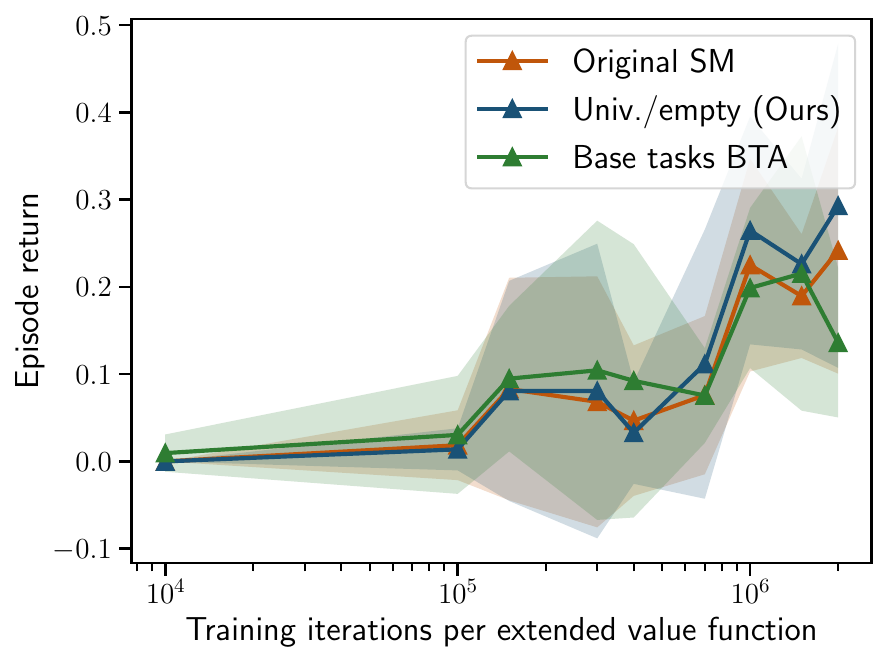}
        \caption{Average episode return.}
        \label{fig:safety_returns}
    \end{subfigure}
    \hfill
    \begin{subfigure}{0.48\linewidth}
        \centering
        \includegraphics[width=\linewidth]{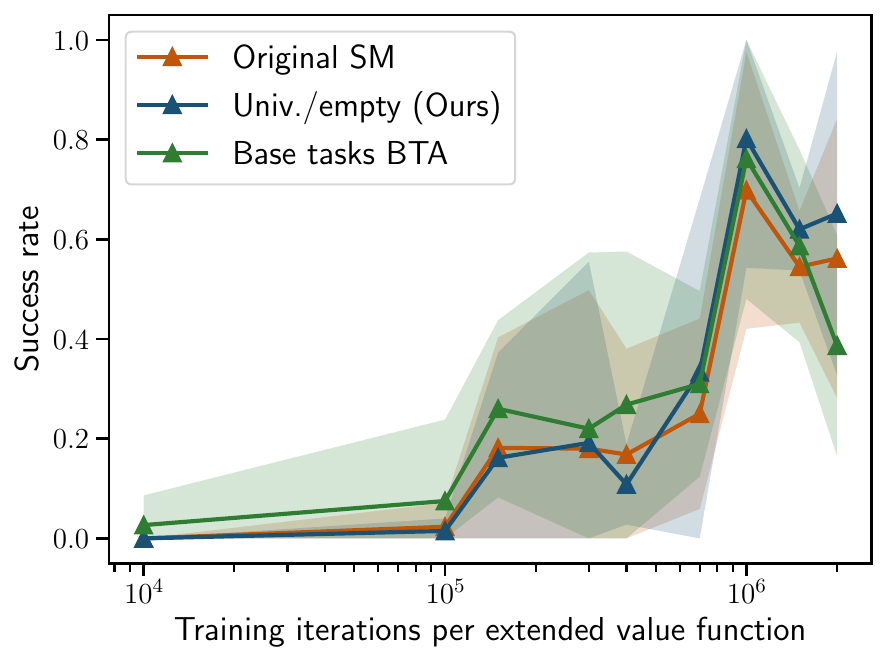}
        \caption{Average success rate.}
        \label{fig:safety_successes}
    \end{subfigure}
    \caption{Evolution of average return and success rate in the Safety Gym environment. Results compare methods for a fixed number of training iterations per UVFA.}
    \label{fig:safety_gym_results}
\end{figure}

\begin{figure}[t]
    \centering
    \includegraphics[width=0.55\linewidth]{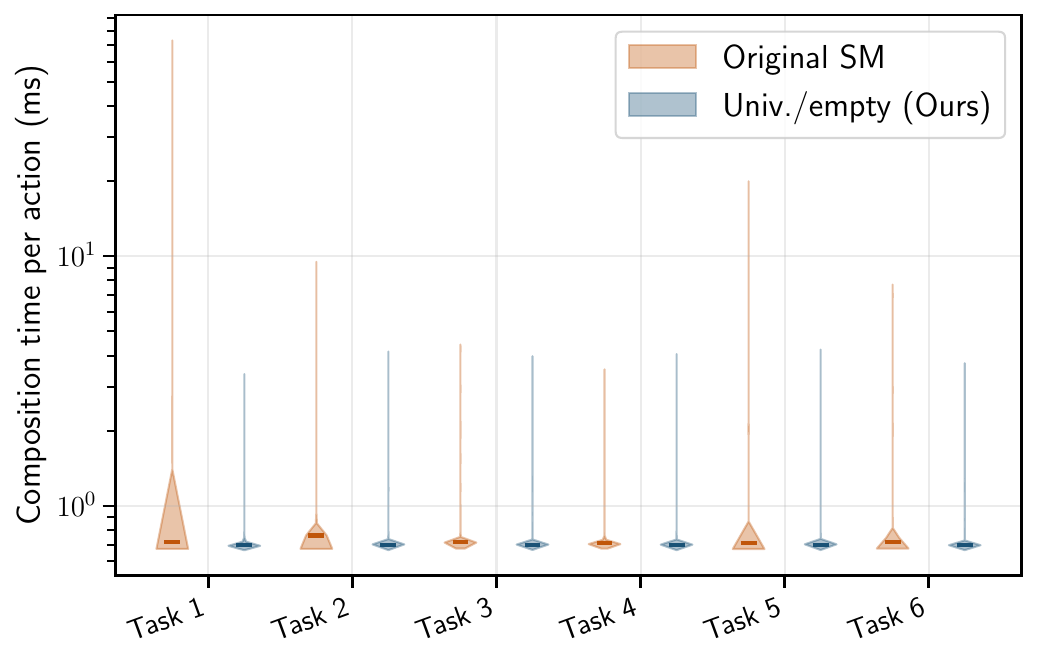}
    \caption{Distribution of composition times in the Safety Gym environment. The Univ./empty composition reduces the computational cost of composing extended value functions compared with the original Skill Machines composition, while maintaining comparable performance as shown in Figure~\ref{fig:safety_gym_results}.
    See Appendix \ref{appendix:b:detailed-experimental-design} for explanation on the LTL task specifications.
    }
    \label{fig:safety_composition_time}
\end{figure}

\end{document}